\documentclass{article}

\usepackage{arxiv}

\usepackage[utf8]{inputenc} 
\usepackage[T1]{fontenc}    
\usepackage{hyperref}       
\usepackage{url}            
\usepackage{booktabs}       
\usepackage{amsfonts}       
\usepackage{nicefrac}       
\usepackage{microtype}      
\usepackage{lipsum}		
\usepackage{graphicx}
\usepackage{natbib}
\usepackage{doi}
\usepackage{}

\usepackage{amsmath,amsfonts,bm}

\def\eqref#1{equation~\ref{#1}}

\def\1{\bm{1}}

\DeclareMathAlphabet{\mathsfit}{\encodingdefault}{\sfdefault}{m}{sl}
\SetMathAlphabet{\mathsfit}{bold}{\encodingdefault}{\sfdefault}{bx}{n}

\usepackage{amsmath}
\usepackage{amsthm}
\usepackage{wrapfig}
\usepackage{algorithm}
\usepackage{algpseudocode}

\usepackage{subfigure}
\usepackage{bm}
\usepackage{amsmath,amsfonts,amsthm}
\usepackage{amssymb}
\usepackage{tikz}
\usepackage{graphicx}
\usepackage{multirow}

\theoremstyle{plain}
\newtheorem{theorem}{Theorem}[section]
\newtheorem{definition}[theorem]{Definition}
\newtheorem{lemma}[theorem]{Lemma}

\newcommand{\abs}{\text{abs}}
\newcommand{\tr}{\text{Tr}}

\newtheorem{example}[theorem]{Example}

\title{Mutual Regression Distance}

\author{ {Dong Qiao} \\
	School of Science and Engineering\\
	The Chinese University of Hong Kong, Shenzhen\\
	\texttt{dongqiao@link.cuhk.edu.cn} \\
	\And
	{Jicong Fan}\thanks{Corresponding author} \\
	School of Data Science\\
	The Chinese University of Hong Kong, Shenzhen\\
	\texttt{fanjicong@cuhk.edu.cn} \\
}

\renewcommand{\shorttitle}{\textit{arXiv} Template}

\hypersetup{
pdftitle={A template for the arxiv style},
pdfsubject={q-bio.NC, q-bio.QM},
pdfauthor={David S.~Hippocampus, Elias D.~Striatum},
pdfkeywords={First keyword, Second keyword, More},
}

\begin{document}
\maketitle

\begin{abstract}
The maximum mean discrepancy and Wasserstein distance are popular distance measures between distributions and play important roles in many machine learning problems such as metric learning, generative modeling, domain adaption, and clustering. However, since they are functions of pair-wise distances between data points in two distributions, they do not exploit the potential manifold properties of data such as smoothness and hence are not effective in measuring the dissimilarity between the two distributions in the form of manifolds. In this paper, different from existing measures, we propose a novel distance called Mutual Regression Distance (MRD) induced by a constrained mutual regression problem, 
which can exploit the manifold property of data.
We prove that MRD is a pseudometric that satisfies almost all the axioms of a metric. 
Since the optimization of the original MRD is costly, we provide a tight MRD and a simplified MRD, based on which a heuristic algorithm is established. 
We also provide kernel variants of MRDs that are more effective in handling nonlinear data. 
Our MRDs especially the simplified MRDs have much lower computational complexity than the Wasserstein distance. We provide theoretical guarantees, such as robustness, for MRDs. 
Finally, we apply MRDs to distribution clustering, generative models, and domain adaptation. 
The numerical results demonstrate the effectiveness and superiority of MRDs compared to the baselines.
\end{abstract}


\section{Introduction}
\label{sec: intro}
Probability distribution is an effective tool to describe certain complex data objects of many domains like document retrieval, image matching, and generative modeling \citep{chizat2020faster,tolstikhin2016minimax,arjovsky2017wasserstein}. A lot of machine learning methods intimately rely on comparison between data pairs \citep{vayer2023controlling,yang2006distance,von2007tutorial,van2008visualizing,kulis2013metric,kaya2019deep}. To compare two probability distributions, some divergences have been proposed like Kullback-Leibler divergence \citep{kullback1951information}, beta divergence \citep{basu1998robust}, and TV divergence \citep{pinsker1964information}. While these classical norms or distances only evaluate the difference of density point-wisely, another kind of distance can capture the geometric nature of probability distributions \citep{feydy2019interpolating}. For example, Wasserstein distance and sinkhorn distance were developed based on the optimal transport (OT) theory \citep{villani2009optimal,cuturi2013sinkhorn}. \citet{gretton2006kernel,gretton2012kernel} presented maximum mean distance (MMD) from the perspective of statistical test. Both OT distances and MMD have been intensively studied in theory and applications \citep{modeste2024characterization,agrawal2021optimal,schrab2023mmd}.

Wasserstein distance, a type of OT metric, measures the difference between probability measures \citep{kantorovich1960mathematical} and is also known as the Earth Mover's distance \citep{rubner2000earth}. To compute the mean of a set of empirical probability measures, \cite{wassersteinbarycenter2011} introduced the concept of a Wasserstein barycenter. Leveraging this metric, \cite{arjovsky2017wasserstein} presented Wasserstein generative adversarial networks (WGANs). Subsequently, \cite{gulrajani2017improved} proposed the WGAN-GP, incorporating a gradient penalty to enforce the Lipschitz constraint more effectively. Despite its strengths, Wasserstein distance is computationally expensive \citep{villani2009optimal}. To address this, \cite{haviv2024wasserstein} proposed Wasserstein Wormhole, which approximates OT distances using Euclidean distances in a latent space, significantly reducing computational costs.

Sinkhorn distance, proposed by \cite{cuturi2013sinkhorn}, extends classical OT distances by adding an entropy regularization term, significantly improving computational efficiency  through Sinkhorn's matrix scaling algorithm. Its sample complexity bounds were established by \cite{genevay2019sample}, providing theoretical guarantees for its use in practical application. To further accelerate computation, \cite{altschuler2017near} proposed the Greenkhorn algorithm, outperforming the classical Sinkhorn algorithm. For generative modeling, \cite{genevay2018learning} developed a scalable approach for training large-scale models by combining entropic smoothing with algorithm differentiation, allowing the loss function to interpolate between OT and MMD. Building on this, \cite{feydy2019interpolating} further introduced the Sinkhorn divergence. Sinkhorn distance now serves a robust and computationally efficient alternative to classical OT.

MMD was initially introduced to determine whether two samples originate from different distributions \citep{gretton2006kernel,gretton2012kernel}. Since then, it has been widely applied in machine learning \citep{dziugaite2015training}. For example, \cite{wang2024spectral} investigated the use of spectral clustering combined with MMD to cluster discrete distributions. In the context of generative modeling, \cite{li2017mmd} introduced MMD-GAN, leveraging adversarial kernel learning for improved GAN training. Later on, both \cite{binkowski2018demystifying} and \cite{arbel2018gradient} developed the gradient penalty techniques for MMD-GAN. Additionally, MMD has proven effective for domain adaptation. For instance, \cite{baktashmotlagh2016distribution} proposed a distribution-matching embedding approach, using MMD to align the source and target distributions. It is worthy of noting that MMD is computationally more efficient than OT distances, offering a practical and cost-effective alternative\citep{feydy2019interpolating}.

The definitions, detailed in Section \ref{sec: preliminaries}, of Wasserstein distance, Sinkhorn distance, and MMD imply that they are functions of dissimilarities between pair-wise points in two distributions, i.e., $f(\{\|\mathbf{x}_i^{\mathcal{D}_k}-\mathbf{x}_j^{\mathcal{D}_l}\|\}_{i,j,k,l})$. For instance, in Wasserstein distance, $f$ is the inner product between the cost matrix $M$ and the transport plan matrix $P$, where the cost matrix $M$ is usually the pair-wise distance matrix between data points. In many applications, data lie on low-dimensional manifolds, while functions on the dissimilarities between data pairs on two manifolds do not exploit the important properties of manifolds such as smoothness, and hence are not effective in measuring the dissimilarity between the two manifolds. In other words, the Wasserstein distance, Sinkhorn distance, and MMD may not be effective enough to quantify the dissimilarity between two distributions in the form of manifolds. 

In this work, different from the aforementioned distance measures, we propose a novel distance called mutual regression distance (MRD). MRD is built upon linear or kernel regression between two sets of samples drawn from two distributions respectively and is able to exploit the potential manifold property of data. The definition is concise and its computation is simple. Similar to MMD and OT distances, MRD can be applied to many problems such as generative modeling, clustering, and transfer learning. Our contribution is summarized as follows.
\begin{itemize}
    \item We present a novel distance metric MRD and provide an effective optimization algorithm for the computation.
    \item We theoretically show that MRD satisfies almost all axioms of metric.
    \item We provide a few variants of MRD such as simplified MRD and kernel MRD.
    \item We analyze the robustness of MRD to perturbations.
\end{itemize}
We apply MRD to synthetic-data distribution transformation, discrete distribution clustering, deep generative modeling, and domain adaptation, in comparison to many competitors. The results demonstrated the effectiveness and superiority of MRD.

\textbf{Notations}\hspace{1em} We use $x$, $\bm{x}$, $\bm{X}$, and $\mathcal{X}$ to denote scalar, vector, matrix, and set (or distribution), respectively and use $\|\cdot\|$, $\|\cdot\|_F$, $\|\cdot\|_2$, and $\|\cdot\|_\infty$ to denote vector's Euclidean norm, matrix's Frobenius norm, matrix's spectral norm (largest singular value), vector or matrix's $\ell_\infty$ norm (maximum absolute element), respectively. We denote by $\sigma_{min}(\bm X)$ the smallest singular value of a matrix $\bm X$. Given a set $I$, we use $|I|$ to denote the cardinality of $I$. If we use $i\in[N]$ for positive integers $i$ and $N$, it means $i\in\{1, 2, \dots, N\}$.

\section{Preliminaries and Related Work}
\label{sec: preliminaries}
In this section, we review some popular distances to evaluate the difference between distributions and their applications. 
\paragraph{Wasserstein distance} For two probability vectors $\bm r$ and $\bm c$ in the simplex $\Sigma_d = \{\bm x\in\mathbb R_+^d; \bm x^T\bm 1_d = 1\}$, the Wasserstein distance \citep{villani2009optimal,panaretos2019statistical} between $\bm r$ and $\bm c$ is defined as the optimum of the problem
\begin{equation*}
\begin{aligned}
d_M(\bm r,\bm c) = \min_{\bm P\in U(\bm r,\bm c)}\langle \bm P, \bm M \rangle
\end{aligned}
\end{equation*}
where $U(\bm r,\bm c) = \{\bm P\in\mathbb R_+^{d\times d}; \bm P\bm 1_d = \bm r,\bm P^T\bm 1_d = \bm c\}$ and $\bm M\in\mathcal{M} = \{\bm M\in\mathbb R_+^{d\times d};\forall i,j\le d, m_{ij} = 0\iff i = j, \forall i,j,k\le d,m_{ij}\le m_{ik} + m_{kj}\}$. $\bm{M}$ is often called the cost matrix in the problem of optimal transport. There are also a few variants or extensions of Wasserstein distance \citep{paty2019subspace,nguyen2022hierarchical}, such as sliced Wasserstein distance \citep{bonneel2015sliced,kolouri2019generalized}, which often focus on improving the computational efficiency.

\paragraph{Sinkhorn distance} For two probability vectors $\bm r$ and $\bm c$ in the simplex $\Sigma_d = \{\bm x\in\mathbb R_+^d; \bm x^T\bm 1_d = 1\}$, the Sinkhorn distance \citep{cuturi2013sinkhorn,feydy2019interpolating} between $\bm r$ and $\bm c$ is defined as the optimum of the problem
\begin{equation*}
\begin{aligned}
d_M(\bm r,\bm c) = \min_{\bm P\in U_\alpha(\bm r,\bm c)}\langle \bm P, \bm M \rangle
\end{aligned}
\end{equation*}
where $U_\alpha(\bm r,\bm c) = \{\bm P\in U(\bm r,\bm c); \text{KL}(\bm P||\bm r\bm c^T)\le \alpha\}\subset U(\bm r,\bm c)$. It was shown by \citep{cuturi2013sinkhorn} that for a large enough $\alpha$, the Sinkhorn distance is equal to the Wasserstein distance. In practice, the constraint of KL divergence is usually relaxed to a regularizer called entropic regularization.

\paragraph{Maximum mean discrepancy} Distinguishing two distributions by finite samples is known as \textit{Two-Sample Test} in statistics. Maximum mean distance (MMD) is exactly a common way to conduct two-sample test \citep{gretton2012kernel}. Given two distributions $P$ and $Q$, the MMD in $\mathcal{H}_k$ between two distributions $P$ and $Q$ over $\mathcal{X}$ associated to a kernel $k$ is defined by
\begin{align*}
\text{MMD}_k^2(P,Q) = \|\mu_P - \mu_Q\|_k^2 = \mathbb{E}_{P\otimes P}[k(\bm{x},\bm{x}')] - 2\mathbb{E}_{P\otimes Q}[k(\bm{x},\bm{y})] + \mathbb{E}_{Q\otimes Q}[k(\bm{y},\bm{y}')]
\end{align*}

In practice, we use finite samples $X = \{\bm{x}_1, \dots, \bm{x}_m\}$ and $Y = \{\bm{y}_1, \dots, \bm{y}_m\}$ from $P$ and $Q$ to estimate the MMD distance. That is,
\begin{align*}
\widehat{\text{MMD}}_k^2(X,Y) = \frac{1}{\binom{m}{2}}\sum_{i = 1}^m\sum_{j\neq i}^m k(\bm{x}_i,\bm{x}_j) - \frac{2}{\binom{m}{2}}\sum_{i = 1}^m\sum_{j = 1}^m k(\bm{x}_i,\bm{y}_j) + \frac{1}{\binom{m}{2}}\sum_{i = 1}^m\sum_{j\neq i}^m k(\bm{y}_i,\bm{y}_j)
\end{align*}

The aforementioned three distances have been applied to many machine learning problems such as generative modeling \citep{masud2023multivariate,arjovsky2017wasserstein,gulrajani2017improved,li2017mmd}, domain adaptation \citep{baktashmotlagh2016distribution,ben2006analysis,courty2016optimal,weiss2016survey,shen2018wasserstein}, and discrete distribution clustering \citep{Li_Wang_2007,ye2017fast}, and graph learning \citep{chen2020optimal,sun2024mmd}.
For instance, \citet{arjovsky2017wasserstein} proposed the Wasserstein generative adversarial networks (WGAN), where the Wasserstein distance was used to compare the distance between the distribution of real data and generated data. \citet{li2017mmd} proposed the MMD-GAN. \citet{shen2018wasserstein} proposed Wasserstein distance-guided representation learning for domain adaption. \cite{Li_Wang_2007} and \cite{ye2017fast} proposed to use Wasserstein barycenter \citep{wassersteinbarycenter2011} to cluster discrete distributions such as images and text. \citet{wang2024spectral} used Wasserstein distance, Sinkhorn distance, or MMD to construct a similarity matrix of discrete distribution and then conduct spectral clustering. \citet{sun2024mmd} proposed a graph kernel based on MMD and graph neural network.

As discussed before, Wasserstein distance, Sinkhorn distance, and MMD do not exploit the potential manifold property of data and hence are not effective in measuring the dissimilarity between the two distributions in the form of manifolds. To exploit the potential manifold structure of the data, we propose mutual regression distance. It is known that, on a manifold, the smoothness property implies that every data point can be well represented as a linear combination of its a few neighbors. This principle has been widely used in manifold learning and data clustering. If the two manifolds are closer to each other, we can represent every point on one manifold using a small number of points on the other manifold more accurately. This inspires us to use the regression error to quantify the dissimilarity between two distributions.

Note that our work is different from subspace distance \citep{sun2007further,ye2016schubert} and manifold alignment \citep{ham2003learning,wang2009manifold}. Subspace distance is usually based on the principal angles or Golub-Werman distance, which results in information loss when the data distributions are complex. Manifold alignment aims to find two projections to map the data points on two manifolds to a new space while preserving the neighborhood relationships. Most manifold alignment algorithms assume the correspondences between data points on two manifolds are known or partially known, which never holds when comparing two distributions. Although the method proposed by \citep{wang2009manifold} does not require known correspondence, it isn't easy to use the method as a distance metric between two distributions.

\section{Mutual Regression Distance}
\label{sec: mrd for spectral clustering}
In this section, we first define a mutual regression problem and then show that it satisfies some axioms of metric. Based on it, we define the mutual regression distance.

\subsection{Mutual regression problem} 

Suppose the columns of $\bm{X}_1 \in \mathbb R^{m\times n_1}$ and $\bm{X}_2 \in \mathbb R^{m\times n_2}$ are drawn from some unknown distributions $\mathcal{D}_1$ and $\mathcal{D}_2$ respectively. We would like to measure how different between $\bm{X}_1$ and $\bm{X}_2$ or even between $\mathcal{D}_1$ and $\mathcal{D}_2$.
First, we define a mutual regression problem (MRP) as follows.
\begin{definition}[Mutual regression problem]\label{def_MRP}
Assume $w_1, w_2 \in \mathbb{R}_+$ such that $w_1 + w_2 = 1$, let $\bm{X}_1 \in \mathbb R^{m\times n_1}$, $\bm{X}_2 \in\mathbb R^{m\times n_2}$, $\bm{S}_{12} \in \mathbb R^{n_2\times n_1}$,
$\bm{S}_{21} \in\mathbb R^{n_1\times n_2}$, and $\mathcal{S}_{2}^{\le 1} = \{\bm{S} \in \mathbb R^{n\times n'}: \|\bm{S}\|_2 \le 1, n\in\mathbb N^*, n'\in\mathbb N^* \}$, then a mutual regression problem can be defined as
\begin{equation}
\begin{aligned}
\mathop{\textup{minimize}}_{\bm{S}_{12},\bm{S}_{21}}~~ & \sqrt{w_1\Vert\bm{X}_1 - \bm{X}_2\bm{S}_{12}\Vert_F^2 + w_2\Vert\bm{X}_2 - \bm{X}_1\bm{S}_{21}\Vert_F^2}\\
\textup{subject to}~~& \{\bm{S}_{12}, \bm{S}_{21}\} \subseteq \mathcal{S}_{2}^{\le 1}
\end{aligned}
\end{equation}
\end{definition}
In MRP, we perform linear regression between the columns of $\bm{X}_1$ and $\bm{X}_2$ mutually, where the coefficients matrices $\bm{S}_1$ and $\bm{S}_2$ are required to have unit spectral norms. The mutual regressions mean that we try to use the columns of $\bm{X}_2$ to represent the columns $\bm{X}_1$ and vice versa. Some simple properties of MRP are as follows.
\begin{lemma}[Feasible set is convex]\label{lemma: feasible set is convex}
Assume $\bm{S}_1,\bm{S}_2 \in \mathcal{S}_{2}^{\le 1}$, and $\lambda \in [0, 1]$, let $\bm{S} = \lambda \bm{S}_1 + (1 - \lambda)\bm{S}_2$, then $\|\bm{S}\|_2 \in \mathcal{S}_{2}^{\le 1}$.
\end{lemma}

\begin{lemma}\label{lemma: MRP is a convex problem}
MRP is a convex optimization problem
\end{lemma}
Without loss of generality, in MRP, for convenience, we let $w_1=w_2=\frac{1}{2}$, and define the mutual regression distance (MRD) as follows.
\begin{definition}[MRD]
With the same notations as Definition \ref{def_MRP}, the MRD between $\bm{X}_1$ and $\bm{X}_2$ is defined as
\begin{equation}
\begin{aligned}
\text{MRD}(\bm{X}_1, \bm{X}_2) = & \min_{\{\bm{S}_{12}, \bm{S}_{21}\} \subseteq \mathcal{S}_{2}^{\le 1}} ~\sqrt{\frac{1}{2}\Vert\bm{X}_1 - \bm{X}_2\bm{S}_{12}\Vert_F^2 + \frac{1}{2}\Vert\bm{X}_2 - \bm{X}_1\bm{S}_{21}\Vert_F^2}
\end{aligned}
\end{equation}
\end{definition}

\begin{theorem}[MRD is a pseudometric]
\label{thm: MRD is a pseudometric}
$\text{MRD}(\cdot, \cdot)$ satisfies all three axioms of pseudometric. Namely, for any $\bm{X}_1, \bm{X}_2, \bm{X}_3 \in \mathbb{R}^{m\times n}$, the following properties hold.
\begin{itemize}
\item Non-negativity: $\text{MRD}(\cdot, \cdot) \ge 0$ and $\text{MRD}(\bm{X}_1, \bm{X}_2) = 0$ if $\bm{X}_1 = \bm{X}_2$
\item Symmetry: $\text{MRD}(\bm{X}_1, \bm{X}_2) = \text{MRD}(\bm{X}_2, \bm{X}_1)$
\item Triangle inequality: $\text{MRD}(\bm{X}_1, \bm{X}_3) \le \text{MRD}(\bm{X}_1, \bm{X}_2) + \text{MRD}(\bm{X}_2, \bm{X}_3)$
\end{itemize}
\end{theorem}
It is worth noting that the theorem does not explicitly take into account the case that there exists a permutation matrix $\bm{P}$ on the columns of $\bm{X}_1$ or $\bm{X}_2$ such that the three axioms hold. Therefore, MRD is permutation-invariant. MRD does not satisfy the separation property, i.e., $d(\bm{X}_1, \bm{X}_2) = 0$ does not imply $\bm{X}_1 = \bm{X}_2$ or $\bm{X}_1 = \bm{X}_2\bm{P}$ for some permutation $\bm{P}$. This, however, is a useful property of MRD. We provide an example.
\begin{example}
    Suppose the columns of $\bm{X}_1$ and $\bm{X}_2$ are drawn from a subspace in $\mathbb{R}^m$, where the orthogonal bases are $\bm{U}=[\bm{u}_1,\ldots,\bm{u}_d]$ and $d\leq\min\{m,n_1,n_2\}$. Suppose $\bm{X}_1=\bm{U}\bm{Z}_1$ and $\bm{X}_2=\bm{U}\bm{Z}_2$. Then $\bm{U}=\bm{X}_2\bm{Z}_2^\top(\bm{Z}_2\bm{Z}_2^\top)^{-1}$. This indicates that $\bm{X}_1=\bm{X}_2\bm{S}_{12}$, where $\bm{S}_{12}=\bm{Z}_2^\top(\bm{Z}_2\bm{Z}_2^\top)^{-1}\bm{Z}_1$. It is possible that the spectral norm of $\bm{S}_{12}$ is less than or equal to $1$ (e.g., 
 both $\|\bm{Z}_2^\top(\bm{Z}_2\bm{Z}_2^\top)^{-1}\|_2$ and $\|\bm{Z}_1\|_2$ are no larger than 1). Similarly, we can find an $\bm{S}_{21}$ such that $\bm{X}_2=\bm{X}_1\bm{S}_{21}$ and $\|\bm{S}_{21}\|_2\leq 1$. These show that even when the columns of $\bm{X}_1$ and $\bm{X}_2$ are different, $\text{MRD}(\bm{X}_1,\bm{X}_2)$ is still possible to be zero.
\end{example}

Note that the constraints on the spectral norms of  $\bm{S}_{12}$ and $\bm{S}_{21}$ can make it inefficient to solve this corresponding problem. To make it smooth, we can tighten the constraints by replacing matrix 2-norm with Frobenius norm because $\Vert\bm{S}\Vert_F \le 1$ implies $\Vert\bm{S}\Vert_2 \le 1$. Therefore, we can consider the following version.
\begin{definition}[Tightened MRD]
Assume $\bm{X}_1 \in \mathbb R^{m\times n_1}$, $\bm{X}_2 \in\mathbb R^{m\times n_2}$, $\bm{S}_{12} \in \mathbb R^{n_2\times n_1}$,
$\bm{S}_{21} \in\mathbb R^{n_1\times n_2}$, and $\mathcal{S}_{F}^{\le 1} = \{\bm{S} \in \mathbb R^{m\times n}: \|\bm{S}\|_F \le 1, m\in\mathbb N^*, n\in\mathbb N^* \}$, the tightened MRD is defined as
\begin{equation}
\begin{aligned}
\text{MRD}_t(\bm{X}_1, \bm{X}_2) = \min_{\{\bm{S}_{12}, \bm{S}_{21}\} \subseteq \mathcal{S}_{F}^{\le 1}} \sqrt{\frac{1}{2}\Vert\bm{X}_1 - \bm{X}_2\bm{S}_{12}\Vert_F^2 + \frac{1}{2}\Vert\bm{X}_2 - \bm{X}_1\bm{S}_{21}\Vert_F^2}
\end{aligned}
\end{equation}
\end{definition}

Obviously, it still satisfies the properties in Theorem \ref{thm: MRD is a pseudometric}. However, in this case, the mutual regression errors may not be sufficiently small especially when $n_1$ and $n_2$ are large. 

Note that in some scenarios such as generative modeling, the distance between two distributions is not necessarily a metric or pseudo-metric. Instead, one only needs to measure the dissimilarity between the source and target distributions. Therefore, we provide the following simplified MRD.
\begin{definition}[Simplified MRD]
Assume $\bm{X}_1 \in \mathbb R^{m\times n_1}$, $\bm{X}_2 \in\mathbb R^{m\times n_2}$, $\bm{S}_{12} \in \mathbb R^{n_2\times n_1}$, and
$\bm{S}_{21} \in\mathbb R^{n_1\times n_2}$, the simplified MRD is defined as
$$\text{MRD}_s(\bm{X}_1, \bm{X}_2) =  \sqrt{\frac{1}{2}\Vert\bm{X}_1 - \bm{X}_2\bm{S}_{12}^\ast\Vert_F^2 + \frac{1}{2}\Vert\bm{X}_2 - \bm{X}_1\bm{S}_{21}^\ast\Vert_F^2}$$
where 
$\bm{S}_{12}^\ast= \mathop{\textup{argmin}}_{\bm{S}_{12}} \frac{1}{2}\Vert\bm{X}_1 - \bm{X}_2\bm{S}_{12}\Vert_F^2+\frac{\lambda_{12}}{2}\|\bm{S}_{21}\|_F^2$ and $ \bm{S}_{21}^\ast= \mathop{\textup{argmin}}_{\bm{S}_{21}}\frac{1}{2}\Vert\bm{X}_2 - \bm{X}_1\bm{S}_{21}\Vert_F^2+\frac{\lambda_{21}}{2}\|\bm{S}_{21}\|_F^2$.   
\end{definition}
The simplified MRD has a closed-form solution, i.e., $\bm{S}_{12}^\ast = (\bm{X}_2^T\bm{X}_2 + \lambda_{12}\bm{I}_{n_2})^{-1}\bm{X}_2^T\bm{X}_1$ and 
$\bm{S}_{21}^\ast = (\bm{X}_1^T\bm{X}_1 + \lambda_{21}\bm{I}_{n_1})^{-1}\bm{X}_1^T\bm{X}_2$. This greatly improves the computational efficiency.

\renewcommand{\algorithmicrequire}{\textbf{Input:}}
\renewcommand{\algorithmicensure}{\textbf{Output:}}

\begin{algorithm}[!ht]
\caption{Heuristic algorithm for searching for feasible $\lambda_{12}$}
\label{alg: heuristic method for lambda}
\begin{algorithmic}[1]
\Require $\bm{X}_1, \bm{X}_2, tol$
\State Initialization: $l = 0, r = \Vert\bm{X}_2^T\bm{X}_1\Vert_2 - \Vert\bm{X}_2^T\bm{X}_2\Vert_2;$
\State $c = (l + r)/2;$
\While{True}
\State $\Vert S_{12}\Vert_2 = \Vert(\bm{X}_2^T\bm{X}_2 + c\bm{I}_{n_2})^{-1}\bm{X}_2^T\bm{X}_1\Vert_2$
\If{$\abs(\Vert S_{12}\Vert_2 - 1) < tol$}
\State break;
\EndIf
\If{$\Vert S_{12}\Vert_2 > 1$}
\State $l = c;$
\Else
\State $r = c;$
\EndIf
\State $c = (l + r)/2;$
\EndWhile
\end{algorithmic}
\textbf{Output: } $\lambda_{12}=c$
\end{algorithm}
  
\subsection{Optimization}

Although the proposed MRD is a convex optimization problem, it may not be that efficient to solve such a constrained problem. However, we can take advantage of  MRD$_t$ and MRD$_s$ to obtain MRD approximately and efficiently. Our idea is searching $\lambda_{12}$ and $\lambda_{21}$ using MRD$_s$ such that $\bm{S}_{12}$ and $\bm{S}_{21}$ have a unit spectral norm, which is detailed in the following context.






The following lemma provides guidance for searching $\lambda_{12}$.
\begin{lemma}[Upper bound on regularization coefficient]\label{lem:upper bdd of regularization coefficient}
Assume $\bm{X}_1 \in \mathbb R^{m\times n_1}$ and $\bm{X}_2 \in\mathbb R^{m\times n_2}$, there exists a unique $\lambda_{12}$ in $[0,r]$ such that 
\begin{equation}
\begin{aligned}
\Vert\bm{S}_{12}(\lambda_{12})\Vert_2 = \Vert(\bm{X}_2^T\bm{X}_2 + \lambda_{12}\bm{I}_{n_2})^{-1}\bm{X}_2^T\bm{X}_1\Vert_2 = 1
\end{aligned}
\end{equation}
where $r = \Vert\bm{X}_2^T\bm{X}_1\Vert_2 - \sigma_{min}(\bm{X}_2^T\bm{X}_2)$.
\end{lemma}

By Lemma \ref{lem:upper bdd of regularization coefficient}, we can construct a sequence of regularized least square estimators $\bm{S}_{12}(\lambda_{12}^{(k)})$, $k\in\mathbb N$ to get a feasible solution to the MRD by the heuristic Algorithm \ref{alg: heuristic method for lambda}, and so is $\bm S_{21}(\lambda_{21})$. Then, we can evaluate the MRD between $\bm{X}_1$ and $\bm{X}_2$ efficiently.

For this algorithm, we have to evaluate the spectral norm of $\bm{S}_{12}(\lambda_{12})$ or $\bm{S}_{21}(\lambda_{21})$ with multiple times. It may be very time-consuming when either $n_1$ or $n_2$ is large. One way to improve it is based on Sherman-Woodbury-Morrison formula if $m \ll \min\{n_1,n_2\}$. Another alternative is to precompute the SVD before the search loop starts. Specifically, assume $\bm{X}_2^T\bm{X}_2 = \bm V\Sigma\bm V^T$, one has
\begin{equation}
\begin{aligned}
\Vert\bm S_{12}\Vert_2 & = \Vert(\bm{X}_2^T\bm{X}_2 + c\bm{I}_{n_2})^{-1}\bm{X}_2^T\bm{X}_1\Vert_2 = \Vert(\bm V\Sigma\bm V^T + c\bm{I}_{n_2})^{-1}\bm{X}_2^T\bm{X}_1\Vert_2\\
& = \Vert[\bm V(\boldsymbol{\Sigma} + c\bm{I}_{n_2})\bm V^T]^{-1}\bm{X}_2^T\bm{X}_1\Vert_2 = \Vert\bm V(\boldsymbol{\Sigma} + c\bm{I}_{n_2})^{-1}\bm V^T\bm{X}_2^T\bm{X}_1\Vert_2
\end{aligned}
\end{equation}
where $(\boldsymbol{\Sigma} + c\bm{I}_{n_2})^{-1}$ is easy to evaluate since both $\boldsymbol{\Sigma}$ and $\bm I_{n_2}$ are diagonal matrices.




\begin{figure}
    \centering
    \includegraphics[width=0.8\linewidth]{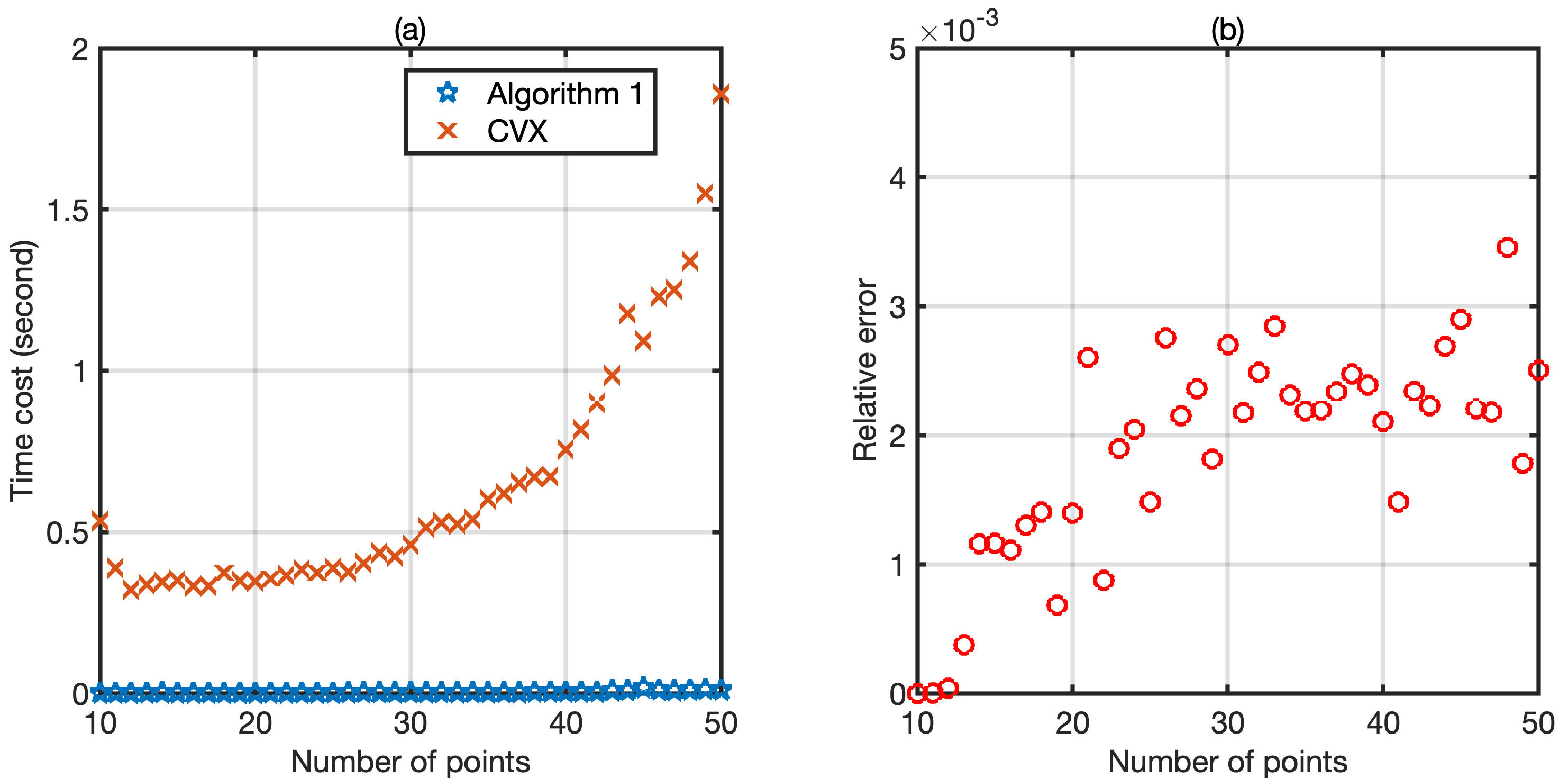}
    \caption{Comparison of constrained optimization solved by CVX and Algorithm 1}
    \label{fig:opt}
\end{figure}

To show the effectiveness of Algorithm \ref{alg: heuristic method for lambda}, we generate a toy example by drawing samples from two Gaussian distributions with different means and the same variance. We compare Algorithm \ref{alg: heuristic method for lambda} with the constrained optimization solved by CVX in MATLAB. The time cost comparison when the number of samples increases from 10 to 50 is shown in Figure \ref{fig:opt}(a). We see that Algorithm \ref{alg: heuristic method for lambda} is much faster than the constrained optimization solved by CVX. Figure \ref{fig:opt}(b) shows the relative error, defined as (MRD$_\text{Alg1}$-MRD$_\text{CVX})/$MRD$_\text{CVX}$. We see that the relative error is always less than $0.5\%$, which demonstrates the high accuracy of Algorithm \ref{alg: heuristic method for lambda}.

\subsection{Kernel MRD}
We can equip the proposed MRD with some kernel functions (e.g., polynomial and Gaussian kernels), which can further augment its discriminative power because the kernel implicitly maps the data into a very high dimensional vector space in which the mapped data may be easier to identify.
\begin{definition}[Kernel MRD]
Using the same notations as Definition \ref{def_MRP} and letting $\phi:\mathbb{R}^m\rightarrow\mathbb{R}^M$ be a feature map induced by a kernel function $\mathcal{K}(\cdot,\cdot)$, the kernel MRD is defined as
\begin{equation}
\begin{aligned}
\text{KMRD}(\bm{X}_1, \bm{X}_2) = & ~\min_{\{\bm{S}_{12}, \bm{S}_{21}\} \subseteq \mathcal{S}_{2}^{\le 1}} ~\sqrt{\frac{1}{2}\Vert\phi(\bm{X}_1) - \phi(\bm{X}_2)\bm{S}_{12}\Vert_F^2 + \frac{1}{2}\Vert\phi(\bm{X}_2) - \phi(\bm{X}_1)\bm{S}_{21}\Vert_F^2}\\
= & ~\min_{\{\bm{S}_{12}, \bm{S}_{21}\} \subseteq \mathcal{S}_{2}^{\le 1}} ~\left\{\frac{1}{2}\textup{Tr}\left(\mathcal{K}(\bm X_1,\bm X_1) - 2\mathcal{K}(\bm X_1,\bm X_2)\bm S_{12} + \bm S_{12}^T\mathcal{K}(\bm X_2,\bm X_2)\bm S_{12}\right)\right.\\
& \left.+ \frac{1}{2}\textup{Tr}\left(\mathcal{K}(\bm X_2,\bm X_2) - 2\mathcal{K}(\bm X_2,\bm X_1)\bm S_{21} + \bm S_{21}^T\mathcal{K}(\bm X_1,\bm X_1)\bm S_{21}\right)\right\}^{1/2}
\end{aligned}
\end{equation}
\end{definition}
Clearly, it still satisfies the properties in Theorem \ref{thm: MRD is a pseudometric}. Analogously, we define a tightened kernel MRD as follows.
\begin{definition}[Tightened Kernel MRD]
Assume $\bm{X}_1 \in \mathbb R^{m\times n_1}$, $\bm{X}_2 \in\mathbb R^{m\times n_2}$, $\bm{S}_{12} \in \mathbb R^{n_2\times n_1}$,
$\bm{S}_{21} \in\mathbb R^{n_1\times n_2}$, and $\mathcal{S}_{F}^{\le 1} = \{\bm{S} \in \mathbb R^{m\times n}; \|\bm{S}\|_F \le 1, m\in\mathbb N^*, n\in\mathbb N^* \}$, the tightened kernel MRD is defined as
\begin{equation}
\begin{aligned}
\text{KMRD}_t(\bm{X}_1, \bm{X}_2) = & \min_{\bm{S}_{1,2}} \sqrt{\frac{1}{2}\Vert\phi(\bm{X}_1) - \phi(\bm{X}_2)\bm{S}_{12}\Vert_F^2 + \frac{1}{2}\Vert\phi(\bm{X}_2) - \phi(\bm{X}_1)\bm{S}_{21}\Vert_F^2}\\
& \text{s.t. }\bm{S}_{1,2} = \{\bm{S}_{12}, \bm{S}_{21}\} \subseteq \mathcal{S}_{F}^{\le 1}
\end{aligned}
\end{equation}
\end{definition}

Similar to the simplified MRD, we can define a simplified kernel MRD.
\begin{definition}[Simplified KMRD]
Assume $\bm{X}_1 \in \mathbb R^{m\times n_1}$, $\bm{X}_2 \in\mathbb R^{m\times n_2}$, $\bm{S}_{12} \in \mathbb R^{n_2\times n_1}$, and
$\bm{S}_{21} \in\mathbb R^{n_1\times n_2}$, the simplified kernel MRD is defined as
$$\text{KMRD}_s(\bm{X}_1, \bm{X}_2) =  \sqrt{\frac{1}{2}\Vert\phi(\bm{X}_1) - \phi(\bm{X}_2)\bm{S}_{12}^\ast\Vert_F^2 + \frac{1}{2}\Vert\phi(\bm{X}_2) - \phi(\bm{X}_1)\bm{S}_{21}^\ast\Vert_F^2}$$
where 
$\bm{S}_{12}^\ast= \mathop{\textup{argmin}}_{\bm{S}_{12}} \frac{1}{2}\Vert\phi(\bm{X}_1) - \phi(\bm{X}_2)\bm{S}_{12}\Vert_F^2+\frac{\lambda_{12}}{2}\|\bm{S}_{21}\|_F^2$ and $ \bm{S}_{21}^\ast= \mathop{\textup{argmin}}_{\bm{S}_{21}}\frac{1}{2}\Vert\phi(\bm{X}_2) - \phi(\bm{X}_1)\bm{S}_{21}\Vert_F^2+\frac{\lambda_{21}}{2}\|\bm{S}_{21}\|_F^2$.   
    
\end{definition}
The simplified MRD has a closed-form solution, i.e., $\bm{S}_{12}^\ast = (\mathcal{K}(\bm{X}_2,\bm{X}_2) + \lambda_{12}\bm{I}_{n_2})^{-1}\mathcal{K}(\bm{X}_2,\bm{X}_1)$ and 
$\bm{S}_{21}^\ast = (\mathcal{K}(\bm{X}_1,\bm{X}_1) + \lambda_{21}\bm{I}_{n_1})^{-1}\mathcal{K}(\bm{X}_1,\bm{X}_2)$. This greatly improves the computational efficiency. Analogously, we can use the heuristic Algorithm \ref{alg: heuristic method for lambda} to help evaluate the kernel MRD between $\bm X_1$ and $\bm X_2$.

Compared to the linear MRD, kernel MRD is able to handle data with more complex structures. The following provides an example of nonlinear data for which the kernel MRD archives zero value.
\begin{example}
    Suppose the columns of $\bm{X}_1$ and $\bm{X}_2$ are generated by an $\alpha$-order polynomial model on a $d$-dimensional latent variable, meaning that they are lying on the same manifold. Let $\mathcal{K}$ be a $q$-order polynomial kernel. Then the ranks of $\phi(\boldsymbol{X}_1)$ and $\phi(\boldsymbol{X}_2)$ are both no larger than $\min \left\{\binom{d+\alpha q}{\alpha q},\binom{m+q}{q}, n\right\}$. Thus we can write $\phi(\bm{X}_1)=\bm{U}\bm{Z}_1$ and $\phi(\bm{X}_2)=\bm{U}\bm{Z}_2$, where $\text{rank}(\bm{U})\leq \binom{d+\alpha q}{\alpha q}$ provided that $n\gg m\gg d$ and $\alpha,q$ are not too large. We have $\phi(\bm{X})_1=\phi(\bm{X})_2\bm{S}_{12}$, where $\bm{S}_{12}=\bm{Z}_2^\top(\bm{Z}_2\bm{Z}_2^\top)^{-1}\bm{Z}_1$. It is possible that the spectral norm of $\bm{S}_{12}$ is less than or equal to $1$. Similar results can be obtained for $\bm{S}_{21}$. In this case, KMRD is zero.
\end{example}

\section{Theoretical Analysis}
\label{sec: theoretical guarantees}
\label{subsec: sensitivity to noise}
As real data are often noisy, a reliable distance measure between distributions shouldn't be too sensitive to noise or perturbations. Therefore, in this section, we investigate the robustness of MRD to perturbations. 

The following theorem shows the robustness of linear MRD to Gaussian noise.
\begin{theorem}[Perturbation analysis on MRD]\label{thm_perturbation_analysis_1}
Given $\bm X_1\in\mathbb R^{m\times n_1}, \bm X_2\in\mathbb R^{m\times n_2}$, let $\Tilde{\bm X}_1 = \bm X_1 + \bm \Delta_1$ and $\Tilde{\bm X}_2 = \bm X_2 + \bm \Delta_2$ be respectively the perturbed data of $\bm X_1$ and $\bm X_2$ where all entries of $\bm \Delta_1 = (e_{i,j}^{(1)})_{i\in[m],j\in[n_1]}$ and $\bm \Delta_2 = (e_{i,j}^{(2)})_{i\in[m],j\in[n_2]}$ are sampled from $\mathcal{N}(0,\sigma^2)$. Then with probability at least $1 - e^{-t}$, the perturbed MRD between $\Tilde{\bm X}_1$ and $\Tilde{\bm X}_2$ is bounded as
\begin{equation}
\begin{aligned}
\left|\text{MRD}(\bm{X}_1 + \bm{\Delta}_1, \bm{X}_2 + \bm{\Delta}_2) - \text{MRD}(\bm{X}_1, \bm{X}_2)\right| \le 2\sigma\xi_{m, n_1,n_2}\sqrt{w_1 + w_2}
\end{aligned}
\end{equation}
where $\xi_{m, n_1,n_2} = \sqrt{m\max\{n_1,n_2\} + 2\sqrt{m\max\{n_1,n_2\}t} + 2t}$.
\end{theorem}

To present the robustness of kernel MRD, we first show the robustness of the kernel matrix to Gaussian noises in the following lemma.
\begin{lemma}[Perturbation error on kernel matrix]\label{lem_perturbation_error_on_kernel_matrix}
Given $\bm X_1\in\mathbb R^{m\times n_1}, \bm X_2\in\mathbb R^{m\times n_2}$, let $\Tilde{\bm X}_1 = \bm X_1 + \bm \Delta_1$ and $\Tilde{\bm X}_2 = \bm X_2 + \bm \Delta_2$ be respectively the perturbed data of $\bm X_1$ and $\bm X_2$ where all entries of $\bm \Delta_1 = (e_{i,j}^{(1)})_{i\in[m],j\in[n_1]}$ and $\bm \Delta_2 = (e_{i,j}^{(2)})_{i\in[m],j\in[n_2]}$ are sampled from $\mathcal{N}(0,\sigma^2)$. Then with the probability at least $1 - n_1n_2e^{-t}$, the perturbation error on the Gaussian kernel matrix is bounded as
\begin{equation}\begin{aligned}
\|\mathcal{K}(\Tilde{\bm X}_1, \Tilde{\bm X}_2) - \mathcal{K}(\bm X_1,\bm X_2)\|_{\infty} \le \frac{1}{r^2}\left[\left(\sigma\xi_m + \frac{\|\mathcal{D}(\bm X_1,\bm X_2)\|_\infty}{\sqrt{2}}\right)^2 - \frac{\|\mathcal{D}(\bm X_1,\bm X_2)\|_\infty^2}{2}\right]
\end{aligned}\end{equation}
where $\xi_m = \sqrt{m + 2\sqrt{mt} + 2t}$, $r$ is the hyperparameter of the Gaussian kernel controlling the smoothness, and $\|\mathcal{D}(\bm X_1,\bm X_2)\|_\infty$ is the maximum entry of the pairwise distance matrix between $\bm X_1$ and $\bm X_2$.
\end{lemma}

\begin{theorem}[Perturbation analysis on kernel MRD]\label{thm_perturbation_analysis_2}
Given $\bm X_1\in\mathbb R^{m\times n_1}, \bm X_2\in\mathbb R^{m\times n_2}$, let $\Tilde{\bm X}_1 = \bm X_1 + \bm \Delta_1$ and $\Tilde{\bm X}_2 = \bm X_2 + \bm \Delta_2$ be respectively the perturbed data of $\bm X_1$ and $\bm X_2$ where all entries of $\bm \Delta_1 = (e_{i,j}^{(1)})_{i\in[m],j\in[n_1]}$ and $\bm \Delta_2 = (e_{i,j}^{(2)})_{i\in[m],j\in[n_2]}$ are sampled from $\mathcal{N}(0,\sigma^2)$. Then with probability at least $1 - (n_1 + n_2)^2e^{-t}$, the perturbed kernel MRD between $\Tilde{\bm X}_1$ and $\Tilde{\bm X}_2$ is bounded as
\begin{equation}
\begin{aligned}
\left|\text{KMRD}(\bm X_1 + \bm\Delta_1,\bm X_2 + \bm\Delta_2) - \text{KMRD}(\bm X_1,\bm X_2)\right| \le \psi_\varepsilon(n_1, n_2)
\end{aligned}
\end{equation}
where $\psi_\varepsilon(n_1, n_2) = 2\sqrt{(w_1 + w_2)\varepsilon n_1 n_2}$, $\varepsilon = \max\{\varepsilon_{11},\varepsilon_{12},\varepsilon_{21},\varepsilon_{22}\}$ for $\varepsilon_{ij} = \frac{1}{r^2}[\big(\sigma\xi_m + \frac{\Vert\mathcal{D}(\bm X_i,\bm X_j)\Vert_\infty}{\sqrt{2}}\big)^2 - \frac{\Vert\mathcal{D}(\bm X_i,\bm X_j)\Vert_\infty^2}{2}]$, and $\xi_m = \sqrt{m + 2\sqrt{mt} + 2t}$.
\end{theorem}

\section{Applications of MRD}
\subsection{Distribution transformation}
Given a dataset $\bm X\in\mathcal{X}$ with the empirical data distribution $P_{\mathcal{X}}$, it is natural to learn an approximation distribution that closely matches the data distribution, enabling efficient sampling from a standard Gaussian distribution $\mathcal{N}(\bm 0,\bm I)$. Substanial progress has been made in probabilistic generative models, such as Autoregressive models \citep{van2016wavenet}, RealNVP \citep{dinh2016density}, and Generative Adversarial Networks \citep{gulrajani2017improved,arbel2018gradient}. Here, we simply focus on a straightforward task where a neural network $f_\theta$ is used to learn an approximation distribution that aligns closely with the target distribution $P_{\mathcal{X}}$, guided by a probability distance metric as the loss function. The details are presented in Algorithm \ref{alg: distribution transformation}.
\begin{algorithm}[!ht]
\caption{Distribution transformation using MRD}
\label{alg: distribution transformation}
\begin{algorithmic}[1]
\Require
$\{\bm{x}_i\}_{i=1}^N$, $\textup{MRD}(\cdot,\cdot)$
\State Initialization: $\theta$, initial generator's parameters; $\alpha$, initial learning rate.
\While{$\theta$ has not converged}
\State Sample a mini-batch $\{\bm x_j\}_{j\in \mathcal{B}}\subseteq\{\bm{x}_i\}_{i=1}^N$
\State Sample a mini-batch $\{\bm z_j\}_{j\in [|\mathcal{B}|]}$ from $\mathcal{N}(\bm 0,\bm I)$
\State $\{\Tilde{\bm x}_j\}_{j\in[|\mathcal{B}|]} = f(\{\bm z_j\}_{j\in [|\mathcal{B}|]};\theta)$
\State $\ell = \textup{MRD}(\{\bm x_j\}_{j\in \mathcal{B}}, \{\Tilde{\bm x}_j\}_{j\in[|\mathcal{B}|]})$
\State $\theta \leftarrow \theta - \alpha\cdot\frac{\partial \ell}{\partial \theta}$
\EndWhile
\Ensure $\theta$.
\end{algorithmic}
\end{algorithm}

\subsection{Discrete distribution spectral clustering using MRD} 
Given a set of distributions $\{\mu_i\}_{i\in[N]}$ in $\mathbb{R}^m$ that can be organized into $K$ groups $\{C_i\}_{i\in[K]}$, suppose $\bm{X}_i=[\bm{x}_1,\bm{x}_2,\ldots,\bm{x}_{n_i}] \in \mathbb{R}^{m\times n_i}$ are independently drawn from $\{\mu_i\}_{i\in[N]}$, the discrete distribution spectral clustering aims at constructing a pairwise distance matrix based on a distance metric $d(\cdot,\cdot)$ and perform vanilla spectral clustering algorithm \citep{shi2000normalized,ng2002spectral,von2007tutorial} on it to partition $\{\bm{X}_i\}_{i\in[N]}$ into $K$ groups corresponding to $\{C_i\}_{i\in[K]}$, respectively \citep{wang2024spectral}. Here, we present the discrete distribution spectral clustering (DDSC) based on MRD in Algorithm \ref{alg: ddsc_mrd}.

\begin{algorithm}[!ht]
\caption{Discrete distribution spectral clustering using MRD}
\label{alg: ddsc_mrd}
\begin{algorithmic}[1]
\Require
$\{\bm{X}_1,\bm{X}_2,\ldots,\bm{X}_N\}$, $\textup{MRD}(\cdot,\cdot)$, $K$, $\tau$
\State Initialization: $\bm{D}=\bm{0}_{N \times N}$
\For{$i=1,2,\ldots,N$}
	\For{$j=i+1,i+2,\ldots,N$}
          \State $D_{ij}=\text{MRD}(\bm{X}_i,\bm{X}_j)$
	\EndFor
\EndFor
\State $\bm{D}=(\bm{D}+\bm{D}^\top)/2$
\State $\bm{A}=\left[\exp(-\gamma D_{ij}^2)\right]_{N\times N}-\bm{I}_N$
\State $\bm{S}=\text{diag}(\sum_{i}A_{i1},\sum_{i}A_{i2},\ldots,\sum_{i}A_{iN})$
\State $\bm{L}=\bm{I}_N-\bm{S}^{-1/2}\bm{A}\bm{S}^{-1/2}$
\State Eigenvalue decomposition: $\bm{L}=\bm{V}\bm{\Lambda}\bm{V}^\top$, where $\lambda_1\leq\lambda_2\cdots\leq \lambda_N$
\State $\bm{V}_K=[\bm{v}_1,\bm{v}_2,\ldots,\bm{v}_K]$
\State Normalize the rows of $\bm{V}_K$ to have unit $\ell_2$ norm.
\State Perform $K$-means on $\bm{V}_K$.
\Ensure $K$ clusters: $C_1,\ldots,C_K$.
\end{algorithmic}
\end{algorithm}

\subsection{SMRDGAN}\label{subsec:mrd gan}
Consider a random variable $X \in \mathcal{X}$ with an empirical data distribution $P_{\mathcal{X}}$ to be learned, our discriminator $D$ processes both real and generated samples, similar to the approach used like SMMDGAN \citep{arbel2018gradient}. It maps these samples into a latent space and maximizes their Kernel MRD. Conversely, the generator $G$ aims to minimize the Kernel MRD between the real and generated samples. Consequently, the loss function of our proposed SMRDGAN is
\begin{equation}
\begin{aligned}
\text{SMRD}: \frac{\textup{KMRD}\left(D(\bm X), D(G(\bm Z))\right)}{1 + 10\mathbb E_{\hat{\mathbb P}}\Vert\nabla D(\bm X)\Vert_F^2}
\end{aligned}
\end{equation}
where $\bm X\sim P_{\mathcal{X}}$, $\bm Z\sim P_{\mathcal{Z}}$. In particular, the denominator incorporates gradient and smooth penalties on $\phi\circ D$, as proposed by \cite{arbel2018gradient}, to enforce the necessary constraint for kernel-based IPM loss.

\subsection{Domain adaptation}
In unsupervised domain adaptation, we are given a source domain $\mathcal{D}_s = \{(\bm x_i^s, y_i^s)\}_{i=1}^{n_s}$ with $n_s$ labeled examples and a target domain $\mathcal{D}_t = \{\bm x_j^t\}_{j=1}^{n_t}$ with $n_t$ unlabeled examples. The source domain and target domain are characterized by probability measures $p_s(\bm x)$ and $q_t(\bm x)$ which typically differ due to the domain shift. Let $f_\theta: \mathcal{X}\rightarrow \mathcal{Z}$ denote the feature extractor (\textit{e.g.}, ResNet-50), which maps input samples to a feature space $\mathcal{Z}$. The resulting feature distributions are denoted as $p_s^f = f_\theta\# p_s$ and $p_t^f = f_\theta\# p_t$ where $f_\theta\#$ denote the pushforward measure induced by the feature extractor. The objective of feature alignment is to minimize the divergence between the source and target domains, \textit{i.e.}, $\min_f d(p_s^f,p_t^f)$ where $d(\cdot,\cdot)$ is a suitable divergence measure (\textit{e.g.}, MRD). For image classification tasks on the unlabeled target domain, this objective is typically combined with a cross-entropy loss (classification loss) on the labeled source domain. It can be formalized by
\begin{equation}
\begin{aligned}
\min_\theta \frac{1}{n_s}\sum_{i=1}^{n_s}J(f_\theta(\bm x_i^s), y_i^s) + \lambda\textup{MRD}(\{\bm x_j^s\}_{j=1}^{n_s}, \{\bm x_j^t\}_{j=1}^{n_t})
\end{aligned}
\end{equation}
where $J(\cdot, \cdot)$ is the classification loss function, and $\lambda\in\mathbb R_{++}$ is the trade-off parameter that balance the classification loss and the domain adaptation loss.

\section{Numerical results}
\label{sec: expr}

\subsection{Distribution transformation}
We compared the discriminative utility of our MRD on synthetic distributions against other distances including Wasserstein distance, Sinkhorn distance, and MMD. Specifically, we used these distances as loss functions in a multilayer perceptron (MLP) to transform a 2D Gaussian distribution into various synthetic target distributions. Details of experimental setup are provided in Appendix \ref{subsec: exp setup}.
\begin{figure}[!ht]
    \centering
    \subfigure[Target]{
    \includegraphics[width=0.16\linewidth]{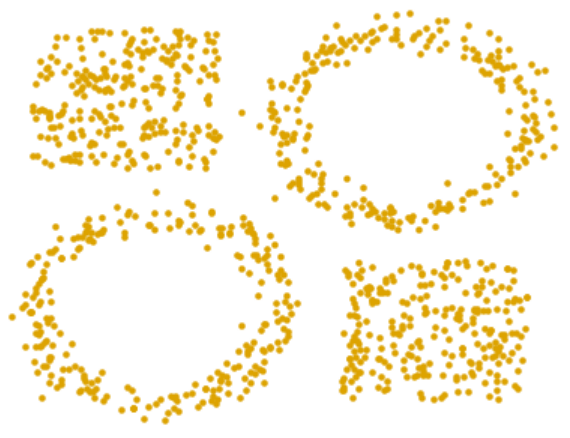}
    \label{fig:toy_gt_cvx}
    }
    \subfigure[Wasserstein]{
    \includegraphics[width=0.16\linewidth]{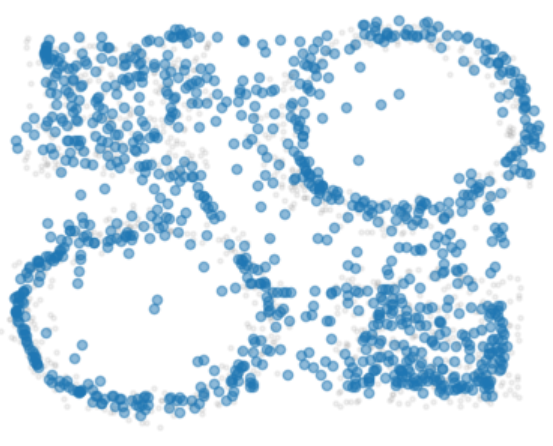}
    \label{fig:toy_emd2_cvx}
    }
    \subfigure[Sinkhorn]{
    \includegraphics[width=0.16\linewidth]{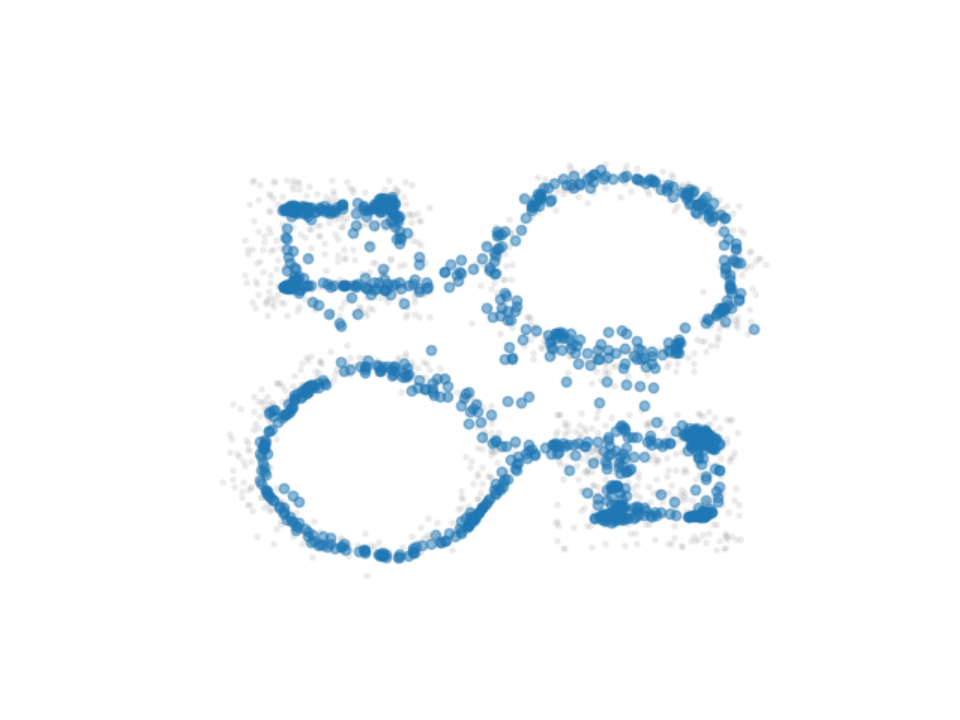}
    \label{fig:toy_sinkhorn_cvx}
    }
    \subfigure[MMD]{
    \includegraphics[width=0.16\linewidth]{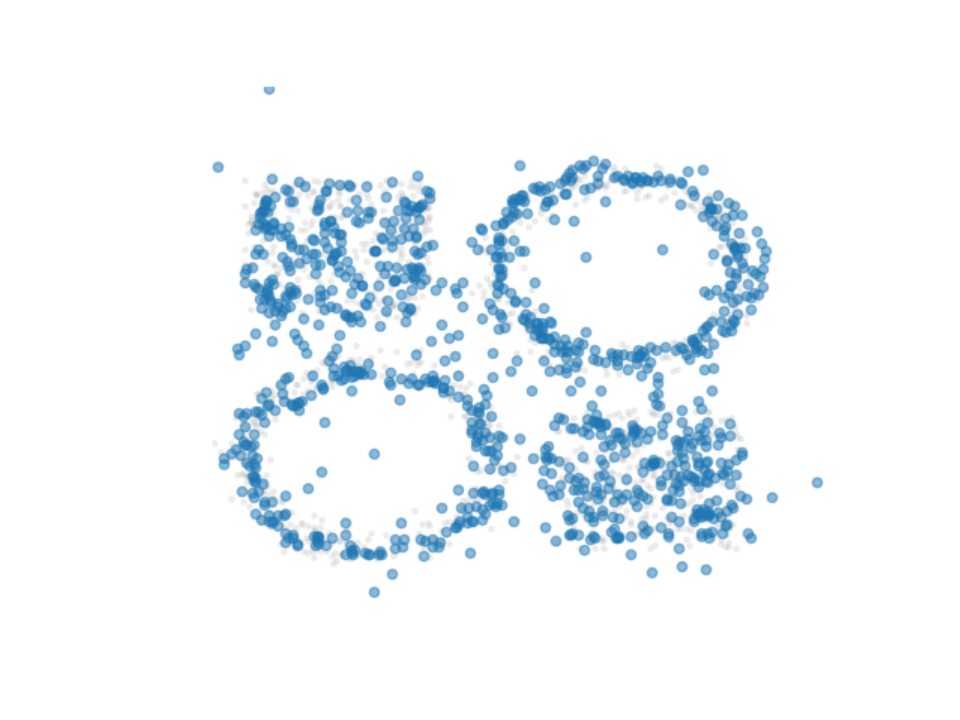}
    \label{fig:toy_mmd_cvx}
    }
    \subfigure[MRD]{
    \includegraphics[width=0.16\linewidth]{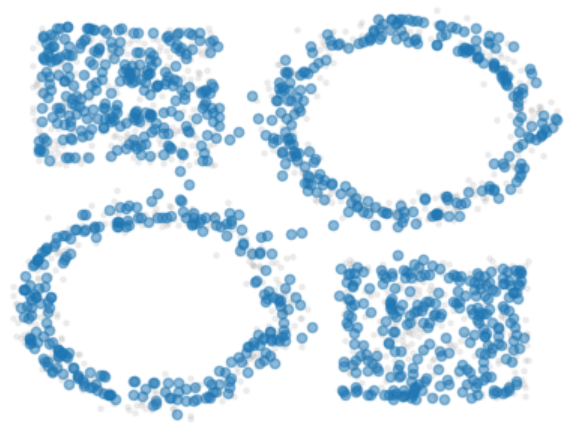}
    \label{fig:toy_mrd_cvx}
    }\par
    \subfigure[Target]{
    \includegraphics[width=0.16\linewidth]{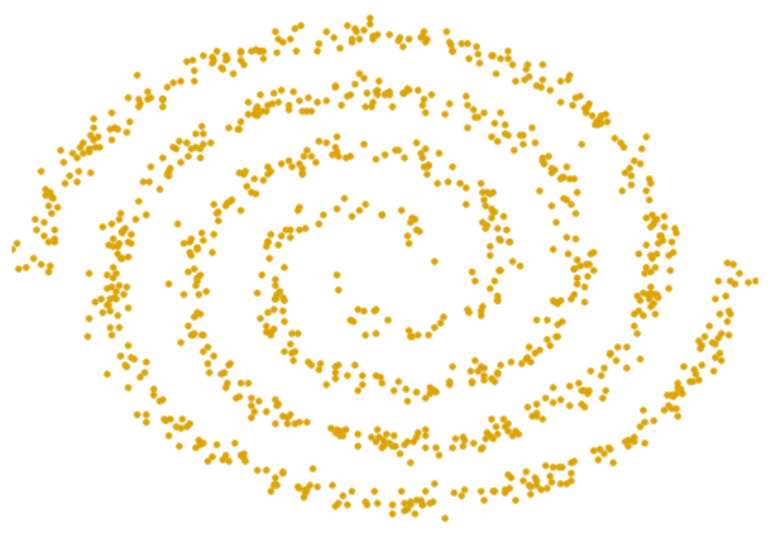}
    \label{fig:toy_gt_spiral}
    }
    \subfigure[Wasserstein]{
    \includegraphics[width=0.16\linewidth]{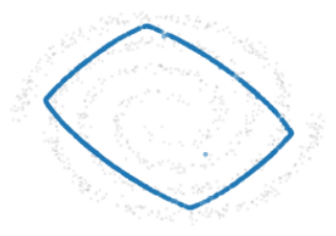}
    \label{fig:toy_emd2_spiral}
    }
    \subfigure[Sinkhorn]{
    \includegraphics[width=0.16\linewidth]{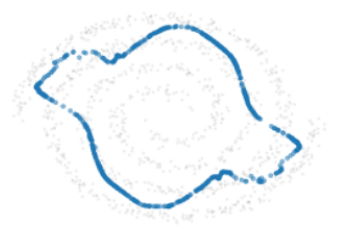}
    \label{fig:toy_sinkhorn_spiral}
    }
    \subfigure[MMD]{
    \includegraphics[width=0.16\linewidth]{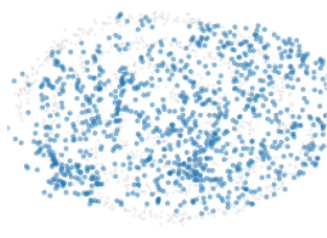}
    \label{fig:toy_mmd_spiral}
    }
    \subfigure[MRD]{
    \includegraphics[width=0.16\linewidth]{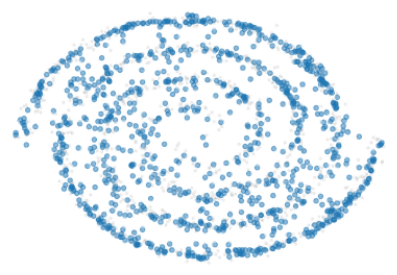}
    \label{fig:toy_mrd_spiral}
    }
    \caption{Distribution transformation on a toy example}
    \label{fig:toy_example}
\end{figure}

As shown in Figure \ref{fig:toy_example}, our proposed MRD outperforms all other distances on the first target distribution that consists of two squares and two circles. However, in the spiral dataset, both the optimal transport distance and the Sinkhorn distance fail to serve as effective loss functions for transforming Gaussian noise into the target distribution. Similarly, MMD did not produce satisfactory results. In contrast, our MRD performed well on this non-convex data distribution, which demonstrated its potential for better capturing and leveraging the manifold properties of complex data.

\subsection{Text clustering}
We checked the utility of our proposed MRD for text clustering on three widely-used text datasets including BBC News, the domain-specific BBC Sports, and a subset of Reuters. Details of them can be found in Appendices \ref{subsec:data_description} and \ref{subsec: exp setup}. Additionally, we compared our method against existing approaches. Specifically, \cite{ye2017fast} and \cite{huang2021projection} introduced fast D2 clustering and PD2 clustering, respectively, while \cite{wang2024spectral} explored the performance of various distances (\textit{i.e.}, Wasserstein distance, Sinkhorn distance, MMD) in DDSC. The latter also proposed using linear optimal transport to construct affinity matrices on large datasets.
\begin{table}[!ht]
\centering
\caption{Clustering results on real text datasets}
\label{tab:clustering comparison}
\begin{tabular}{l|cccccccc}
\toprule
 &
  \multicolumn{2}{c}{BBC-Sports abstr.} &
  \multicolumn{2}{c}{BBCNews abstr.} &
  \multicolumn{2}{c}{Reuters Subsets} &
  \multicolumn{2}{c}{Average Score} \\
Methods             & AMI    & ARI    & AMI    & ARI    & AMI    & ARI    & AMI    & ARI    \\ \midrule
K-means             & 0.3408 & 0.3213 & 0.5328 & 0.4950 & 0.4783 & 0.4287 & 0.4506 & 0.4150 \\
K-means$^*$         & 0.4276 & -      & 0.3877 & -      & 0.4627 & -      & 0.4260 & -      \\
SC                  & 0.3646 & 0.2749 & 0.4891 & 0.4659 & 0.3955 & 0.3265 & 0.4164 & 0.3558 \\
D2                  & 0.6234 & 0.4665 & 0.6111 & 0.5572 & 0.4244 & 0.3966 & 0.5530 & 0.4734 \\
D2$^*$              & 0.6510 & -      & 0.6095 & -      & 0.4200 & -      & 0.5602 & -      \\
PD2                 & 0.6300 & 0.4680 & 0.6822 & 0.6736 & 0.4958 & 0.3909 & 0.6027 & 0.5108 \\
PD2$^*$             & 0.6892 & -      & 0.6557 & -      & 0.4713 & -      & 0.6054 & -      \\ \midrule
DDSC$_{\text{MMD}}$ & 0.6724 & 0.5399 & 0.7108 & 0.6479 & 0.5803 & 0.5105 & 0.6545 & 0.5661 \\
DDSC$_{\text{Sinkhorn}}$ &
  \textbf{0.7855} &
  0.7514 &
  0.7579 &
  0.7642 &
  0.6096 &
  0.5457 &
  0.7177 &
  0.6871 \\
DDSC$_{\text{W}}$   & 0.7755 & 0.7424 & 0.7549 & 0.7585 & 0.6096 & 0.5457 & 0.7133 & 0.6802 \\
DDSC$_{\text{LOT}}$ & 0.7150 & 0.6712 & 0.7265 & 0.7499 & 0.5290 & 0.4325 & 0.6580 & 0.6129 \\ \midrule
DDSC$_{\text{MRD}}$ &
  0.7668 &
  \textbf{0.8078} &
  \textbf{0.7777} &
  \textbf{0.7795} &
  \textbf{0.6385} &
  \textbf{0.6324} &
  \textbf{0.7277} &
  \textbf{0.7399} \\ \bottomrule
\end{tabular}%
\end{table}

From Table \ref{tab:clustering comparison}, one can easily observe that DDSC-class methods consistently outperform other approaches, such as like K-means, spectral clustering, PD2. Notably, our proposed MRD further enhances the performance of DDSC with significant improvements observed from the Adjusted Rand Index (ARI).

\subsection{Image generation}
We evaluated the effectiveness of SMRDGAN on the image generation task and compared it with baselines, including SMMDGAN \citep{arbel2018gradient} and WGAN-GP \citep{gulrajani2017improved}. More details about the experimental setup can be found in Appendix \ref{subsec: exp setup}. Figure \ref{fig: image gen} showcases sample images generated by three GAN models. Additionally, we report the inception score and FID metrics in Table \ref{tab: image generation} for all three models on the MNIST, Fashion-MNIST, CIFAR10, and CelebA datasets. More generated samples of MNIST and Fashion-MNIST can be found in Appendix \ref{subsec: more generated samples}.

\begin{figure}[!ht]
    \centering
    \subfigure[WGAN-GP]{
    \includegraphics[width=0.3\linewidth]{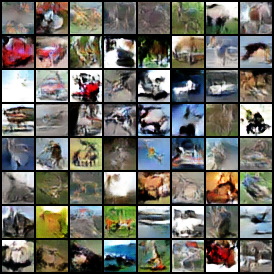}
    \label{fig: wgan-gp cifar10}
    }
    \subfigure[SMMDGAN]{\includegraphics[width=0.3\linewidth]{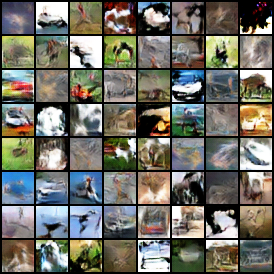}
    \label{fig: mmd cifar10}
    }
    \subfigure[SMRDGAN]{
    \includegraphics[width=0.3\linewidth]{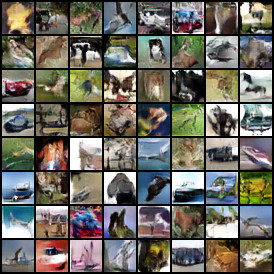}
    \label{fig: mrd cifar10}
    }\par
    \subfigure[WGAN-GP]{
    \includegraphics[width=0.3\linewidth]{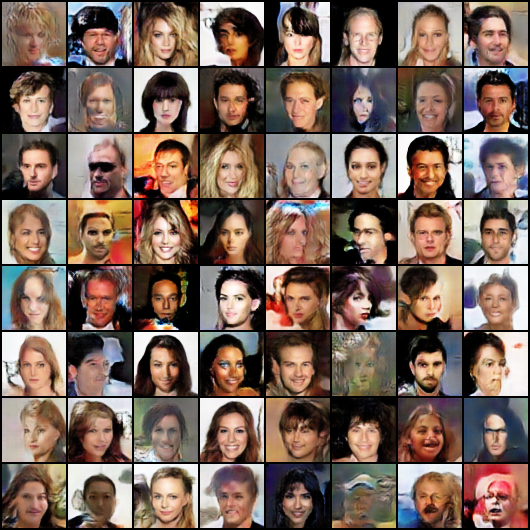}
    \label{fig: wgan-gp celeba}
    }
    \subfigure[SMMDGAN]{
    \includegraphics[width=0.3\linewidth]{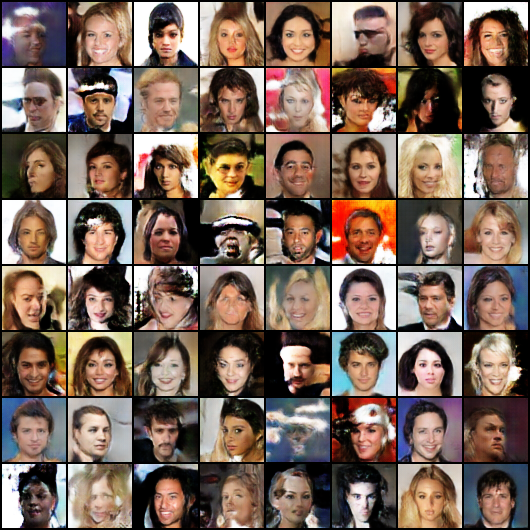}
    \label{fig: mmd celeba}
    }
    \subfigure[SMRDGAN]{
    \includegraphics[width=0.3\linewidth]{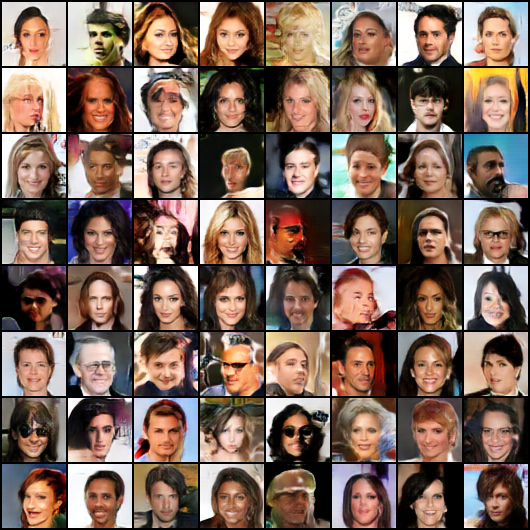}
    \label{fig: mrd celeba}
    }
    \caption{Samples from WGAN-GP, SMMDGAN, and our SMRDGAN. Top: $32\times 32$ CIFAR-10; bottom: $64\times 64$ CelebA.}
    \vspace{-10pt}
    \label{fig: image gen}
\end{figure}

\begin{table}[!ht]
\centering
\caption{Mean (Standard deviation) of score estimators, based on 10000 samples from each model}
\label{tab: image generation}
\resizebox{\textwidth}{!}{%
\begin{tabular}{l|llllllll}
\toprule
        & \multicolumn{2}{l}{MNIST}        & \multicolumn{2}{l}{Fashion-MNIST} & \multicolumn{2}{l}{CIFAR10}      & \multicolumn{2}{l}{CelebA}       \\
Method  & IS $\uparrow$ & FID $\downarrow$ & IS $\uparrow$  & FID $\downarrow$ & IS $\uparrow$ & FID $\downarrow$ & IS $\uparrow$ & FID $\downarrow$ \\ \midrule
WGAN-GP &
  $\mathbf{2.1\pm 0.0}$ &
  $9.9\pm 0.6$ &
  $\mathbf{4.0\pm 0.1}$ &
  $24.3\pm 1.0$ &
  $5.3\pm 0.1$ &
  $59.2\pm 2.1$ &
  $\mathbf{2.5\pm 0.0}$ &
  $38.6\pm 2.4$ \\
SMMDGAN & $2.1\pm 0.0$  & $7.7\pm 0.4$     & $3.9\pm 0.1$   & $22.7\pm 0.7$    & $5.4\pm 0.2$  & $57.6\pm 1.7$    & $2.5\pm 0.0$  & $37.6\pm 3.0$    \\
SMRDGAN &
  $2.0\pm 0.0$ &
  $\mathbf{7.6\pm 1.4}$ &
  $4.0\pm 0.1$ &
  $\mathbf{16.9\pm 0.8}$ &
  $\mathbf{6.2\pm 0.1}$ &
  $\mathbf{46.5\pm 1.4}$ &
  $2.5\pm 0.1$ &
  $\mathbf{25.3\pm 1.7}$ \\ \bottomrule
\end{tabular}%
}
\end{table}

From Table \ref{tab: image generation}, it is evident that our SMRDGAN consistently outperforms both SMMDGAN and WGAN-GP in most cases across three benchmark datasets. Moreover, the proposed SMRDGAN achieved comparable performance to the existing models on MNIST, Fashion-MNIST, and CelebA in terms of the inception score. These results highlight that our proposed MRD serves as a robust alternative to optimal transport distances and MMD for IPM-based GAN models. As our focus is on evaluating the effectiveness of MRD, MMD, and OT distances, we did not include comparisons between SMRDGAN and diffusion models. More work may be done in the future.

\subsection{Domain adaptation}
We evaluated the utility of MRD as a matching measure for feature alignment in domain adaptation on Office-31 classification. The details of the experimental setup can be found in Appendix \ref{sec: more details of exp}. We present the classification result in Table \ref{tab:da on office-31}. It can be found that all metrics take themself into effect in such a standard task. Our MRD outperforms all compared metrics, especially significantly for $A \rightarrow W$ and $A \rightarrow D$. It is well known that the other four transfer tasks including $W\rightarrow A, D\rightarrow A, D\rightarrow W, W\rightarrow D$ are more challenging, but MRD is comparable to those metrics in these tasks. In total, our MRD is more powerful in such transfer learning tasks.

\begin{table}[!ht]
\centering
\caption{Accuracy on Office-31 of unsupervised domain adaptation with different metrics}
\label{tab:da on office-31}
\resizebox{\textwidth}{!}{%
\begin{tabular}{c|ccccc}
\toprule
                 & Source only        & Wasserstein        & Sinkhorn           & MMD                & MRD                \\ \midrule
$A\rightarrow W$ & $0.7996\pm 0.0048$ & $0.9009\pm 0.0055$ & $0.9152\pm 0.0045$ & $0.8765\pm 0.0057$ & $\bm{0.9187}\pm 0.0028$ \\
$W\rightarrow A$ & $0.6753\pm 0.0022$ & $0.7235\pm 0.0051$ & $0.7184\pm 0.0034$ & $0.7042\pm 0.0042$ & $\bm{0.7395}\pm 0.0039$ \\
$A\rightarrow D$ & $0.8301\pm 0.0082$ & $0.8622\pm 0.0058$ & $0.8618\pm 0.0074$ & $0.8588\pm 0.0038$ & $\bm{0.9102}\pm 0.0039$ \\
$D\rightarrow A$ & $0.6743\pm 0.0037$ & $0.7108\pm 0.0058$ & $0.7029\pm 0.0086$ & $0.6960\pm 0.0035$ & $\bm{0.7498}\pm 0.0062$ \\
$D\rightarrow W$ & $0.9823\pm 0.0017$ & $0.9493\pm 0.0052$ & $0.9541\pm 0.0048$ & $0.9800\pm 0.0023$ & $\bm{0.9870}\pm 0.0016$ \\
$W\rightarrow D$ & $0.9986\pm 0.0017$ & $0.9954\pm 0.0014$ & $0.9960\pm 0.0023$ & $0.9974\pm 0.0010$ & $\bm{1.0000}\pm 0.0000$ \\ \midrule
Average          & $0.8267\pm 0.0020$ & $0.8570\pm 0.0025$ & $0.8581\pm 0.0020$ & $0.8521\pm 0.0018$ & $\bm{0.8842}\pm 0.0015$ \\ \bottomrule
\end{tabular}%
}
\end{table}

\section{Conclusion}
\label{sec: conclusion}

This work proposed a novel distance measure called MRD between two distributions. As constrained optimization is costly, we proposed tightened MRD, simplified MRD, and a heuristic algorithm for MRD, which is more efficient than the CVX solver. We also present kernel versions of MRD such as simplified KMRD that is able to hand data with nonlinear structures. We studied the robustness of MRDs theoretically. The experiments of distribution transformation on synthetic data, discrete distribution clustering, generative modeling, and domain adaptation demonstrated the effectiveness of our method in comparison to the baselines. 


\subsubsection*{Acknowledgments}
This work was supported by the National Natural Science Foundation of China under Grant No.62376236. The authors declare that they have no known competing financial interests or personal relationships that could have appeared to influence the work reported in this paper.

\bibliographystyle{unsrtnat}
\bibliography{references}  






\appendix
\section{More details on numerical experiments}\label{sec: more details of exp}
\subsection{Datasets}\label{subsec:data_description}
\textbf{Synthetic datasets} are used to intuitively show the effectiveness of the proposed method. The first dataset consists of some simple convex polygons:  two circles and two squares. Assume that we sample $N$ points and that each polygon includes $n=N/4$ points, this dataset can be generated as follows.

\begin{itemize}
\item Step 1: Sample $\bm\theta\in\mathbb R^{n\times 1}$ from $\textup{Uniform}(0,1)$, set $\bm c = 2\pi\bm\theta$, and generate the lower-left circle by $\mathbb R^{n\times 2} \ni \bm X_1 = (r\cos(\bm c) + \bm e, r\sin(\bm c) + \bm e')$ where $\bm e,\bm e'\in\mathbb R^{n\times 1}$ are sampled from $\mathcal{N}(0,0.01)$;
\item Step 2: Sample $\bm\theta\in\mathbb R^{n\times 1}$ from $\textup{Uniform}(0,1)$, set $\bm c = 2\pi\bm\theta$, and generate the upper-right circle by $\mathbb R^{n\times 2} \ni \bm X_2 = (r\cos(\bm c) + 2 + \bm e, r\sin(\bm c) + 2 + \bm e')$ where $\bm e,\bm e'\in\mathbb R^{n\times 1}$ are sampled from $\mathcal{N}(0,0.01)$;
\item Step 3: Sample $\bm x,\bm x'\in\mathbb R^{n\times 2}$ from $\textup{Uniform}(0,1)$ and generate the lower-right square by $\mathbb R^{n\times 2} \ni \bm X_3 = (1.5\bm x + 1.5, 1.5\bm x' - 1)$;
\item Step 4: Sample $\bm x,\bm x'\in\mathbb R^{n\times 2}$ from $\textup{Uniform}(0,1)$ and generate the upper-left square by $\mathbb R^{n\times 2} \ni \bm X_4 = (1.5\bm x - 1, 1.5\bm x' + 1.5)$;
\end{itemize}

The second one is a spiral dataset. Assume that we sample $N$ points and that each spiral curve has $n=N/2$ points, it can be generated by 
\begin{itemize}
\item Step 1: Sample $\bm\theta\in\mathbb R^{n\times 1}$ from $\textup{Uniform}(0,1)$ and set $\bm c = 4\pi\sqrt{\bm\theta}$;
\item Step 2: Calculate $\bm r_1 = 2\bm c + \pi$ and generate one of the spiral curves by $\bm X_1 = (\bm r_1\cos(\bm c) + \bm e,\bm r_1\sin(\bm c) + \bm e')$ where $\bm e,\bm e'\in\mathbb R^{n\times 1}$ are sampled from $\mathcal{N}(0,0.64)$;
\item Step 3: Calculate $\bm r_2 = -2\bm c - \pi$ and generate the other spiral curve by $\bm X_2 = (\bm r_2\cos(\bm c) + \bm e,\bm r_2\sin(\bm c) + \bm e')$ where $\bm e,\bm e'\in\mathbb R^{n\times 1}$ are sampled from $\mathcal{N}(0,0.64)$;
\end{itemize}

\textbf{BBC News abstract dataset} is created by concatenating the title and the first sentence of news posts from BBC News \citep{greene06icml} comprised of 2225 articles, each labeled under one of 5 categories: business, entertainment, politics, sport or tech. For each article, we retain the first 16 high-frequency words, each of which is converted into a 300-dimensional vector by some natural language processing tools (\textit{e.g.}, NLTK, and Gensim).

\textbf{BBC Sports abstract dataset} is a domain-specific dataset \citep{greene06icml} that consists of 737 documents with 5 classes and also comes from BBC and is similarly constructed as the BBC News abstract dataset.

\textbf{Reuters subset} is a 5-class subset of the Reuters from Reuters financial news services \citep{reuters-21578_text_categorization_collection_137}. We preprocess it as the BBC News abstract dataset.

\textbf{MNIST, Fashion-MNIST, CIFAR10, and CelebA} are the common datasets for machine learning tasks \citep{lecun1998gradient,xiao2017fashion,krizhevsky2009learning, liu2015faceattributes}.

\textbf{Office-31} is a computer vision classification dataset with 31 classes from \cite{saenko2010adapting} with images from three domains: amazon (A), webcam (W), and dslr (D). It is a \textit{de facto} standard for domain adaptation algorithms in computer vision. Amazon, the largest domain, is a composition of 2817 images. Webcam and dslr have 795 and 498 images, respectively.

\subsection{Experimental setup}
\label{subsec: exp setup}
\textbf{Distribution transformation} is based on two synthetic datasets, each with 1024 data points. We use a multilayer perceptron (MLP) to transform a 2-D Gaussian distribution into the synthetic target distributions. This MLP consists of four fully connected linear layers, each followed by the hyperbolic tangent activation function except the last one. That is, the MLP has the structure: $\textup{Linear}(2,nz)\rightarrow\textup{Tanh}\rightarrow\textup{Linear}(nz,nz)\rightarrow\textup{Tanh}\rightarrow\textup{Linear}(nz,nz)\rightarrow\textup{Tanh}\rightarrow\textup{Linear}(nz,2)$. In experiment, we set $nz=255$ to ensure MLP has enough capacity for all compared metrics.

\textbf{Text clustering} is based on three text datasets: BBC News abstract dataset, BBC Sports abstract dataset, and Reuters subset. We perform DDSC on those text datasets. As described in Algorithm \ref{alg: ddsc_mrd}, the computation of Gaussian kernel occurs in both MRD and spectral clustering. We adaptively estimate the 
bandwidth $\sigma$ by averaging all entries of the pairwise distance matrix in both scenarios. That is, we do not have to fine-tune those hyperparameters for improving the performance of DDSC. The results of other methods are directly referenced from \cite{wang2024spectral}.

\textbf{Image generation} is based on four image datasets: MNIST, Fashion-MNIST, CIFAR-10, and CelebA. For MNIST, Fashion-MNIST, and CIFAR-10, we used a 4-layer DCGAN discriminator and a 6-layer ResNet generator. For CelebA, we used a 5-layer DCGAN discriminator and a 8-layer ResNet generator (See \ref{subsec: network arch}). The input codes for the generator are drawn from a standard Gaussian $\mathcal{N}(\bm 0, \bm I_{128})$. The output dimension of the discriminators across all GAN models was set to 1. All models were trained for 10,000 generator updates on a single GPU, using a batch size of 128 and 5 critic updates per generator step. The initial learning rate was set to 0.0001 for all datasets. As in previous literature, we used Adam optimizer with $\beta_1 = 0.5$ and $\beta_2 = 0.9$.

\textbf{Domain adaptation} is based on Office-31. We adopt the unified transfer learning framework developed by \cite{transferlearning.xyz}, which leverages a pretrained ResNet50 as the backbone and employs an MLP classifier for unsupervised domain adaptation. This framework is standard in deep transfer learning \citep{long2015learning,zhu2020deep,long2017deep}. For a fair comparison, we integrated MMD, Wasserstein distance, Sinkhorn distance, and our proposed MRD into their Python loss function module. Additionally, we enhance our MRD using a multi-kernel technique, following the approach outlined in \cite{long2015learning}. Consistent with prior work, we evaluate performance across six common transfer tasks: $A\rightarrow W$, $W\rightarrow A$, $A\rightarrow D$, $D\rightarrow A$, $D\rightarrow W$, and $W\rightarrow D$. To ensure optimal performance, we fine-tuned the trade-off parameter $\lambda$ for all four distances.

\subsection{Evaluation metrics}

\textbf{Fr\'echet inception distance (FID)} is a common metric to assess the quality of images generated by a generative model, like a generative adversarial network (GAN) \citep{heusel2017gans}. Given a collection of real images $(N_1,C,H,W)$ and another collection of fake images $(N_2,C,H,W)$, the FID score between them is computed as follows.
\begin{itemize}
\item Step 1: transform all the images to the shape of $(3, 299, 299)$ which is required as the standard input of the inception v3;
\item Step 2: feed the collection of real images (\textit{resp.}, fake images) into inception v3 so as to get the corresponding collection of feature vectors $(N_1,2048)$ (\textit{resp.}, $(N_2, 2048)$);
\item Step 3: compute $\bm \mu_1$ (\textit{resp.}, $\bm \mu_2$) as the feature-wise mean of the feature vectors of real images (\textit{resp.}, fake images) and also calculate $\bm C_1$ and $\bm C_2$ the covariance matrix for those two collections of feature vectors;
\item Step 4: Calculate the FID score between two collections of feature vectors
\begin{align*}
d^2 = \Vert\bm \mu_1 - \bm \mu_2\Vert_2^2 + \textup{Tr}\left(\bm C_1 + \bm C_2 - 2(\bm C_1\bm C_2)^{1/2}\right)
\end{align*}
\end{itemize}

In our image generative task, we use the test sets from MNIST, Fashion-MNIST, CIFAR10, and CelebA as collections of real images. The generator produces an equal number of fake images, forming a corresponding collection of generated images. Then, we calculate the FID score between the real and generated image sets, assessing the performance of the generator. For this computation, we employed the TorchMetrics library \citep{falcon2019pytorch}.


\subsection{More generated samples}
\label{subsec: more generated samples}
We presented more samples generated by WGAN-GP, SMMDGAN, and our SMRDGAN trained on MNIST and Fashion-MNIST in Figure \ref{fig:gan - more samples}.
\begin{figure}[!ht]
    \centering
    \subfigure[WGAN-GP]{
    \includegraphics[width=0.3\linewidth]{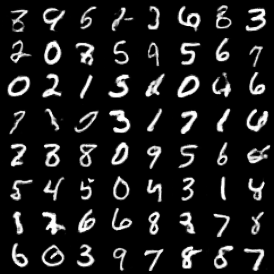}
    \label{fig: wgan-gp mnist}
    }\hfill
    \subfigure[SMMDGAN]{\includegraphics[width=0.3\linewidth]{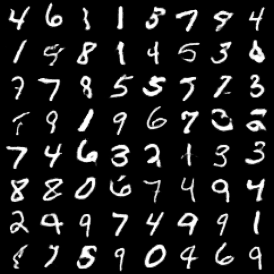}
    \label{fig: mmd mnist}
    }\hfill
    \subfigure[SMRDGAN]{
    \includegraphics[width=0.3\linewidth]{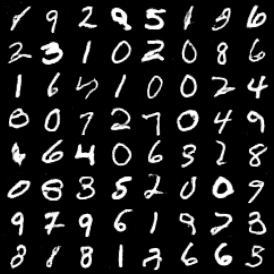}
    \label{fig: mrd mnist}
    }
    \subfigure[WGAN-GP]{
    \includegraphics[width=0.3\linewidth]{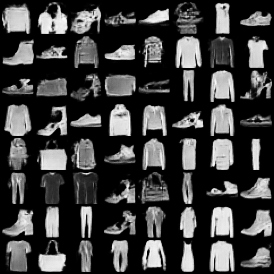}
    \label{fig: wgan-gp fmnist}
    }\hfill
    \subfigure[SMMDGAN]{\includegraphics[width=0.3\linewidth]{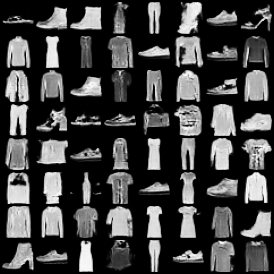}
    \label{fig: mmd fmnist}
    }\hfill
    \subfigure[SMRDGAN]{
    \includegraphics[width=0.3\linewidth]{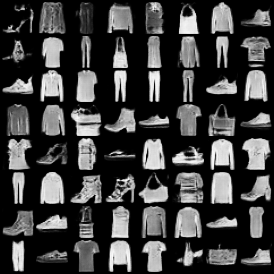}
    \label{fig: mrd fmnist}
    }
    \caption{Samples from WGAN-GP, SMMDGAN, and our SMRDGAN. Top: $32\times 32$ MNIST; bottom: $32\times 32$ Fashion-MNIST.}
    \label{fig:gan - more samples}
    \vspace{-10pt}
\end{figure}

\subsection{Network architectures}
\label{subsec: network arch}
For the image generation tasks on CelebA, we used the following network architectures. For MNIST, Fashion-MNIST, and CIFAR-10, we simplify those two network architectures accordingly.
\begin{table}[!ht]
\centering
\caption{ResNet Generator}
\label{tab:ResNetGenerator}
\begin{tabular}{@{}lccccc@{}}
\toprule
\textbf{Layer Type}  & \textbf{Kernel Size} & \textbf{Stride} & \textbf{Padding} & \textbf{Output Channels} & \textbf{Output Shape} \\ \midrule
Input                & --                  & --              & --               & --                       & $128$ \\ 
Unflatten                & --                  & --              & --               & 128                       & $128 \times 1 \times 1$ \\ 
Deconv 1                & $4 \times 4$        & 1               & 0                & 512                       & $512 \times 4 \times 4$ \\ 
Residual Block 1     & $4 \times 4$        & 2               & 1                & 256                       & $256 \times 8 \times 8$ \\ 
Residual Block 2     & $4 \times 4$        & 2               & 1                & 128                      & $128 \times 16 \times 16$ \\ 
Residual Block 3     & $4 \times 4$        & 2               & 1                & 64                      & $64 \times 32 \times 32$ \\ 
Deconv 2 & $4\times 4$                  & 2              & 1               & 3                     & $3 \times 64 \times 64$ \\ \bottomrule
\end{tabular}
\end{table}

\begin{table}[!ht]
\centering
\caption{DCGAN Discriminator}
\label{tab:DCDiscriminator}
\begin{tabular}{@{}lccccc@{}}
\toprule
\textbf{Layer Type}  & \textbf{Kernel Size} & \textbf{Stride} & \textbf{Padding} & \textbf{Output Channels} & \textbf{Output Shape} \\ \midrule
Input                & --                  & --              & --               & 3                       & $3 \times 64 \times 64$ \\ 
Conv 1                & $4 \times 4$        & 2               & 1                & 64                       & $64 \times 32 \times 32$ \\ 
Conv 2     & $4 \times 4$        & 2               & 1                & 128                       & $128 \times 16 \times 16$ \\ 
Conv 3     & $4 \times 4$        & 2               & 1                & 256                      & $256 \times 8 \times 8$ \\ 
Conv 4     & $4 \times 4$        & 2               & 1                & 512                      & $512 \times 4 \times 4$ \\ 
Conv 5 & $4\times 4$                  & 1              & 0               & ndod                     & $ndod \times 1 \times 1$ \\ 
Flatten & --                  & --              & --               & --                     & $ndod$ \\ \bottomrule
\end{tabular}
\end{table}

It is worth noting that batch normalization is not employed in the discriminator of WGAN-GP \citep{gulrajani2017improved}, whereas SMMDGAN in \cite{arbel2018gradient} utilizes spectral normalization to its convolutional layers. In our proposed SMRDGAN, we incorporate batch normalization on the top of convolutional layers, preceding the Leaky ReLU activation. These strategies are thoughtfully considered to enhance the performance of their respective models.

\section{Proof for Lemma \ref{lemma: feasible set is convex}}
\begin{proof}
Assume $\bm{S}_1,\bm{S}_2 \in \mathcal{S}_{2}^{\le 1}$, and $\lambda \in [0, 1]$, let $\bm{S} = \lambda \bm{S}_1 + (1 - \lambda)\bm{S}_2$, then it has

\begin{equation}
\begin{aligned}
\|\bm{S}\|_2 & = \|\lambda \bm{S}_1 + (1 - \lambda)\bm{S}_2\|_2\\
& \le \lambda\|\bm{S}_1\|_2 + (1 - \lambda)\|\bm{S}_2\|_2\\
& \le 1
\end{aligned}
\end{equation}
\end{proof}

\section{Proof for Lemma \ref{lemma: MRP is a convex problem}}
\begin{proof}
Since the feasible set is convex, we only need to show that the objective function is convex. Consider the objective function as 
\begin{equation}
\begin{aligned}
g(\bm{S}_{12}, \bm{S}_{21}) & = f(\bm{X}_1, \bm{X}_2)\\
& = \sqrt{w_1\Vert\bm{X}_1 - \bm{X}_2\bm{S}_{12}\Vert_F^2 + w_2\Vert\bm{X}_2 - \bm{X}_1\bm{S}_{21}\Vert_F^2}
\end{aligned}
\end{equation}

\vspace{1em}

\begin{equation}
\begin{aligned}
g(\bm{S}_{12}, \bm{S}_{21}) & = \sqrt{\left\Vert\left[\sqrt{w_1}(\bm{X}_1 - \bm{X}_2\bm{S}_{12}),\sqrt{w_2}(\bm{X}_2 - \bm{X}_1\bm{S}_{21})\right]\right\Vert_F^2}\\
& = \left\Vert\left[\sqrt{w_1}(\bm{X}_1 - \bm{X}_2\bm{S}_{12}),\sqrt{w_2}(\bm{X}_2 - \bm{X}_1\bm{S}_{21})\right]\right\Vert_F\\
& = \left\Vert\left[\sqrt{w_1}\bm{X}_1, \sqrt{w_2}\bm{X}_2\right] - \left[\sqrt{w_1}\bm{X}_2\bm{S}_{12}, \sqrt{w_2}\bm{X}_1\bm{S}_{21}\right]\right\Vert_F\\
& = \left\Vert\left[\sqrt{w_1}\bm{X}_1, \sqrt{w_2}\bm{X}_2\right] - \left[\sqrt{w_1}\bm{X}_2, \sqrt{w_2}\bm{X}_1\right]\begin{bmatrix}
\bm{S}_{12} & \bm{O}\\
\bm{O} & \bm{S}_{21}
\end{bmatrix}\right\Vert_F
\end{aligned}
\end{equation}

Thus, we have

\begin{equation}
\begin{aligned}
& g(\lambda\bm{S}_{12} + (1 - \lambda)\bm{S}_{12}', \lambda\bm{S}_{21} + (1 - \lambda)\bm{S}_{21}')\\
= & \left\Vert\left[\sqrt{w_1}\bm{X}_1, \sqrt{w_2}\bm{X}_2\right] - \left[\sqrt{w_1}\bm{X}_2, \sqrt{w_2}\bm{X}_1\right]\begin{bmatrix}
\lambda\bm{S}_{12} + (1 - \lambda)\bm{S}_{12}' & \bm{O}\\
\bm{O} & \lambda\bm{S}_{21} + (1 - \lambda)\bm{S}_{21}'
\end{bmatrix}\right\Vert_F\\
\le & \lambda\left\Vert\left[\sqrt{w_1}\bm{X}_1, \sqrt{w_2}\bm{X}_2\right] - \left[\sqrt{w_1}\bm{X}_2, \sqrt{w_2}\bm{X}_1\right]\begin{bmatrix}
\bm{S}_{12} & \bm{O}\\
\bm{O} & \bm{S}_{21}
\end{bmatrix}\right\Vert_F\\
& + (1 - \lambda)\left\Vert\left[\sqrt{w_1}\bm{X}_1, \sqrt{w_2}\bm{X}_2\right] - \left[\sqrt{w_1}\bm{X}_2, \sqrt{w_2}\bm{X}_1\right]\begin{bmatrix}
\bm{S}_{12}' & \bm{O}\\
\bm{O} & \bm{S}_{21}'
\end{bmatrix}\right\Vert_F\\
= & \lambda g(\bm{S}_{12}, \bm{S}_{21}) + (1 - \lambda)g(\bm{S}_{12}', \bm{S}_{21}')
\end{aligned}
\end{equation}

Therefore, it is a convex optimization.
\end{proof}

\section{Proof for Theorem \ref{thm: MRD is a pseudometric}}

\begin{proof}

(1) ($\Leftarrow$) If $\bm{X}_1 = \bm{X}_2$, 
\begin{equation}
\begin{aligned}
d(\bm{X}_1, \bm{X}_1) = & \min_{\bm{S}_{1,1}} \sqrt{w_1\Vert\bm{X}_1 - \bm{X}_1\bm{S}_{11}\Vert_F^2 + w_2\Vert\bm{X}_1 - \bm{X}_1\bm{S}_{11}\Vert_F^2}\\
& \text{s.t. }\bm{S}_{1,1} = \{\bm{S}_{11}\} \subseteq \mathcal{S}_{2}^{\le 1}
\end{aligned}
\end{equation}

It is clear that $d(\bm{X}_1, \bm{X}_2) = d(\bm{X}_1, \bm{X}_1) = d(\bm{X}_2, \bm{X}_2) = 0$ at $\bm{S}_{11} = I$.

($\Rightarrow$) If $d(\bm{X}_1, \bm{X}_2) = 0$, it has

\begin{equation}
\begin{aligned}
\bm{X}_1 - \bm{X}_2\bm{S}_{12} = 0\\
\bm{X}_2 - \bm{X}_1\bm{S}_{21} = 0\\
\end{aligned}
\end{equation}

which implies that $\bm{X}_1 \sim \bm{X}_2$ which can be indistinguishable.

(2) \begin{equation}
\begin{aligned}
d(\bm{X}_1, \bm{X}_2) = & \min_{\bm{S}_{1,2}} \sqrt{w_1\Vert\bm{X}_1 - \bm{X}_2\bm{S}_{12}\Vert_F^2 + w_2\Vert\bm{X}_2 - \bm{X}_1\bm{S}_{21}\Vert_F^2}\\
& \text{s.t. }\bm{S}_{1,2} = \{\bm{S}_{12}, \bm{S}_{21}\} \subseteq \mathcal{S}_{2}^{\le 1}
\end{aligned}
\end{equation}

\begin{equation}
\begin{aligned}
d(\bm{X}_2, \bm{X}_1) = & \min_{\bm{S}_{2,1}} \sqrt{w_1\Vert\bm{X}_2 - \bm{X}_1\bm{S}_{21}\Vert_F^2 + w_2\Vert\bm{X}_1 - \bm{X}_2\bm{S}_{12}\Vert_F^2}\\
& \text{s.t. }\bm{S}_{2,1} = \{\bm{S}_{21}, \bm{S}_{12}\} \subseteq \mathcal{S}_{2}^{\le 1}
\end{aligned}
\end{equation}

When $w_1 = w_2 = 1/2$, both of them are the same optimization problem. Hence, $d(\bm{X}_1, \bm{X}_2) = d(\bm{X}_2, \bm{X}_1)$ for equally distributed weights.

\vspace{1em}

(3) \begin{equation}
\begin{aligned}
& d(\bm{X}_1, \bm{X}_3) \le f(\bm{X}_1, \bm{X}_3)\\
= & \sqrt{w_1\Vert\bm{X}_1 - \bm{X}_3\bm{S}_{13}\Vert_F^2 + w_2\Vert\bm{X}_3 - \bm{X}_1\bm{S}_{31}\Vert_F^2}\\
= & \sqrt{\Vert\left[\sqrt{w_1}(\bm{X}_1 - \bm{X}_3\bm{S}_{13}),\sqrt{w_2}(\bm{X}_3 - \bm{X}_1\bm{S}_{31})\right]\Vert_F^2} = \Vert\left[\sqrt{w_1}(\bm{X}_1 - \bm{X}_3\bm{S}_{13}),\sqrt{w_2}(\bm{X}_3 - \bm{X}_1\bm{S}_{31})\right]\Vert_F\\
= & \Vert\left[\sqrt{w_1}(\bm{X}_1 - \bm{X}_2\bm{S}_{12} + \bm{X}_2\bm{S}_{12} - \bm{X}_3\bm{S}_{13}),\sqrt{w_2}(\bm{X}_3 - \bm{X}_2\bm{S}_{32} + \bm{X}_2\bm{S}_{32} - \bm{X}_1\bm{S}_{31})\right]\Vert_F\\
= & \Vert\left[\sqrt{w_1}(\bm{X}_1 - \bm{X}_2\bm{S}_{12}),\sqrt{w_2}(\bm{X}_2\bm{S}_{32} - \bm{X}_1\bm{S}_{31})\right] + \left[\sqrt{w_1}(\bm{X}_2\bm{S}_{12} - \bm{X}_3\bm{S}_{13}),\sqrt{w_2}(\bm{X}_3 - \bm{X}_2\bm{S}_{32})\right]\Vert_F \textcolor{red}{*}\\
\le & \Vert\left[\sqrt{w_1}(\bm{X}_1 - \bm{X}_2\bm{S}_{12}),\sqrt{w_2}(\bm{X}_2\bm{S}_{32} - \bm{X}_1\bm{S}_{31})\right]\Vert_F \\
& \quad + \Vert\left[\sqrt{w_1}(\bm{X}_2\bm{S}_{12} - \bm{X}_3\bm{S}_{13}),\sqrt{w_2}(\bm{X}_3 - \bm{X}_2\bm{S}_{32})\right]\Vert_F\\
= & \sqrt{\Vert\left[\sqrt{w_1}(\bm{X}_1 - \bm{X}_2\bm{S}_{12}),\sqrt{w_2}(\bm{X}_2\bm{S}_{32} - \bm{X}_1\bm{S}_{31})\right]\Vert_F^2} \\
& \quad + \sqrt{\Vert\left[\sqrt{w_1}(\bm{X}_2\bm{S}_{12} - \bm{X}_3\bm{S}_{13}),\sqrt{w_2}(\bm{X}_3 - \bm{X}_2\bm{S}_{32})\right]\Vert_F^2}\\
= & \sqrt{\Vert\left[\sqrt{w_1}(\bm{X}_1 - \bm{X}_2\bm{S}_{12}),\sqrt{w_2}(\bm{X}_2\bm{S}_{32} - \bm{X}_1\bm{S}_{21}\bm{S}_{32})\right]\Vert_F^2} \\
& \quad + \sqrt{\Vert\left[\sqrt{w_1}(\bm{X}_2\bm{S}_{12} - \bm{X}_3\bm{S}_{23}\bm{S}_{12}),\sqrt{w_2}(\bm{X}_3 - \bm{X}_2\bm{S}_{32})\right]\Vert_F^2}\\
= & \sqrt{w_1\Vert\bm{X}_1 - \bm{X}_2\bm{S}_{12}\Vert_F^2 + w_2\Vert\bm{X}_2\bm{S}_{32} - \bm{X}_1\bm{S}_{21}\bm{S}_{32}\Vert_F^2} \\
& \quad + \sqrt{w_1\Vert\bm{X}_2\bm{S}_{12} - \bm{X}_3\bm{S}_{23}\bm{S}_{12}\Vert_F^2 + w_2\Vert\bm{X}_3 - \bm{X}_2\bm{S}_{32}\Vert_F^2}
\end{aligned}
\end{equation}
where we use the following lemma for the penultimate identity.

\begin{lemma}[Gluing lemma with norm constraint]
Assume $\bm{S}_{12}, \bm{S}_{21} \bm{S}_{23}, \bm{S}_{32} \in \mathcal{S}_{2}^{\le 1}$, it has\\
\begin{equation}
\bm{S}_{13} = \bm{S}_{23}\bm{S}_{12} \in \mathcal{S}_{2}^{\le 1} \text{ and } \bm{S}_{31} = \bm{S}_{21}\bm{S}_{32} \in \mathcal{S}_{2}^{\le 1}
\end{equation}
\end{lemma}

\vspace{1em}

\begin{proof}

\begin{equation}
\begin{aligned}
\Vert\bm{S}_{13}\Vert_2 & = \Vert\bm{S}_{23}\bm{S}_{12}\Vert_2\\
& \le \Vert\bm{S}_{23}\Vert_2 \Vert\bm{S}_{12}\Vert_2\\
& \le 1
\end{aligned}
\end{equation}

Similarly, $\Vert\bm{S}_{31}\Vert_2 \le 1$. They both imply that $\bm{S}_{13}, \bm{S}_{31} \in \mathcal{S}_{2}^{\le 1}$.

\end{proof}

To derive the final triangle inequality, we use two facts here.
\begin{equation}
\begin{aligned}
& \Vert\bm{X}_2\bm{S}_{32} - \bm{X}_1\bm{S}_{21}\bm{S}_{32}\Vert_F^2 \le \Vert\bm{X}_2 - \bm{X}_1\bm{S}_{21}\Vert_F^2\Vert\bm{S}_{32}\Vert_2^2 \le \Vert\bm{X}_2 - \bm{X}_1\bm{S}_{21}\Vert_F^2\\
& \Vert\bm{X}_2\bm{S}_{12} - \bm{X}_3\bm{S}_{23}\bm{S}_{12}\Vert_F^2 \le \Vert\bm{X}_2 - \bm{X}_3\bm{S}_{23}\Vert_F^2\Vert\bm{S}_{12}\Vert_2^2 \le \Vert\bm{X}_2 - \bm{X}_3\bm{S}_{23}\Vert_F^2
\end{aligned}
\end{equation}
where we use the fact that $\|\bm{A}\bm{B}\|_F \le \|\bm{A}\|_F\|\bm{B}\|_2$ and $f(x) = x^2$ is non-decreasing. Furthermore, since $\|\bm{S}\|_F \le 1 \implies \|\bm{S}\|_2 = \sigma_1(\bm{S}) \le 1$, our constraints are of a relaxation but cannot be relaxed further otherwise we cannot derive the triangle inequality.

Thus, it has
\begin{equation}
\begin{aligned}
d(\bm{X}_1, \bm{X}_3) & \le \sqrt{w_1\Vert\bm{X}_1 - \bm{X}_2\bm{S}_{12}\Vert_F^2 + w_2\Vert\bm{X}_2\bm{S}_{32} - \bm{X}_1\bm{S}_{21}\bm{S}_{32}\Vert_F^2} \\
& \quad + \sqrt{w_1\Vert\bm{X}_2\bm{S}_{12} - \bm{X}_3\bm{S}_{23}\bm{S}_{12}\Vert_F^2 + w_2\Vert\bm{X}_3 - \bm{X}_2\bm{S}_{32}\Vert_F^2}\\
& \le \sqrt{w_1\Vert\bm{X}_1 - \bm{X}_2\bm{S}_{12}\Vert_F^2 + w_2\Vert\bm{X}_2 - \bm{X}_1\bm{S}_{21}\Vert_F^2} \\
& \quad + \sqrt{w_1\Vert\bm{X}_2 - \bm{X}_3\bm{S}_{23}\Vert_F^2 + w_2\Vert\bm{X}_3 - \bm{X}_2\bm{S}_{32}\Vert_F^2}\\
\end{aligned}
\end{equation}
where we got the second inequality by using the above two inequalities and the non-decreasing property of $f(x) = \sqrt{x}$ for this second inequality.

Since $\{\bm{S}_{12}, \bm{S}_{21}, \bm{S}_{23}, \bm{S}_{32}\}$ are all arbitrary, we have
\begin{equation}
\begin{aligned}
d(\bm{X}_1, \bm{X}_3) & \le \inf_{\substack{\|\bm S_{12}\|_2\le 1\\\|\bm S_{21}\|_2\le 1}}\sqrt{w_1\Vert\bm{X}_1 - \bm{X}_2\bm{S}_{12}\Vert_F^2 + w_2\Vert\bm{X}_2 - \bm{X}_1\bm{S}_{21}\Vert_F^2}\\
& \quad + \inf_{\substack{\|\bm S_{23}\|_2\le 1\\\|\bm S_{32}\|_2\le 1}}\sqrt{w_1\Vert\bm{X}_2 - \bm{X}_3\bm{S}_{23}\Vert_F^2 + w_2\Vert\bm{X}_3 - \bm{X}_2\bm{S}_{32}\Vert_F^2}\\
& = d(\bm{X}_1, \bm{X}_2) + d(\bm{X}_2, \bm{X}_3)
\end{aligned}
\end{equation}

\end{proof}

\section{Proof for Theorem \ref{thm: optimality condition of MRD}}

\begin{theorem}\label{thm: optimality condition of MRD}
The optimality condition of MRD$_t$ has a form of
\begin{align}
& \bm{S}_{12}(\lambda_{12}) = (\bm{X}_2^T\bm{X}_2 + \lambda_{12}\bm{I}_{n_2})^{-1}\bm{X}_2^T\bm{X}_1\\
& \bm{S}_{21}(\lambda_{21}) = (\bm{X}_1^T\bm{X}_1 + \lambda_{21}\bm{I}_{n_1})^{-1}\bm{X}_1^T\bm{X}_2\\
& \Vert\bm{S}_{12}(\lambda_{12})\Vert_F^2 = 1,\qquad \Vert\bm{S}_{21}(\lambda_{21})\Vert_F^2 = 1
\end{align}
\end{theorem}

\begin{proof}

The dual objective function is 
\begin{equation}\label{mdl: dual function}
\begin{aligned}
d(\lambda_{12}, \lambda_{21}) = \inf_{\bm{S}_{1,2}} &\left\{\frac{1}{2}\Vert\bm{X}_1 - \bm{X}_2\bm{S}_{12}\Vert_F^2 + \frac{1}{2}\Vert\bm{X}_2 - \bm{X}_1\bm{S}_{21}\Vert_F^2\right.\\
& \left.+ \frac{\lambda_{12}}{2}(\Vert\bm{S}_{12}\Vert_F^2 - 1) + \frac{\lambda_{21}}{2}(\Vert\bm{S}_{21}\Vert_F^2 - 1)\right\}
\end{aligned}
\end{equation}
which is an unconstrained least square problem with the following solution.
\begin{equation}
\begin{aligned}
& \bm{S}_{12}(\lambda_{12}) = (\bm{X}_2^T\bm{X}_2 + \lambda_{12}\bm{I}_{n_2})^{-1}\bm{X}_2^T\bm{X}_1\\
& \bm{S}_{21}(\lambda_{21}) = (\bm{X}_1^T\bm{X}_1 + \lambda_{21}\bm{I}_{n_1})^{-1}\bm{X}_1^T\bm{X}_2
\end{aligned}
\end{equation}
where we consider optimization variables $\bm{S}_{12}$ and $\bm{S}_{21}$ as functions of the regularization parameter $\lambda_{12}$ and $\lambda_{21}$. Hence, it has
\begin{equation}
\begin{aligned}
d(\lambda_{12}, \lambda_{21}) = &\frac{1}{2}\Vert\bm{X}_1 - \bm{X}_2\bm{S}_{12}(\lambda_{12})\Vert_F^2 + \frac{1}{2}\Vert\bm{X}_2 - \bm{X}_1\bm{S}_{21}(\lambda_{21})\Vert_F^2\\
& + \frac{\lambda_{12}}{2}(\Vert\bm{S}_{12}(\lambda_{12})\Vert_F^2 - 1) + \frac{\lambda_{21}}{2}(\Vert\bm{S}_{21}(\lambda_{21})\Vert_F^2 - 1)\\
= & \frac{1}{2}\tr(\bm{X}_1^T\bm{X}_1 - 2\bm{X}_1^T\bm{X}_2\bm{S}_{12}(\lambda_{12}) + \bm{S}_{12}^T(\lambda_{12})\bm{X}_2^T\bm{X}_2\bm{S}_{12}(\lambda_{12}))\\
& + \frac{1}{2}\tr(\bm{X}_2^T\bm{X}_2 - 2\bm{X}_2^T\bm{X}_1\bm{S}_{21}(\lambda_{21}) + \bm{S}_{21}^T(\lambda_{21})\bm{X}_1^T\bm{X}_1\bm{S}_{21}(\lambda_{21}))\\
& + \frac{\lambda_{12}}{2}(\tr(\bm{S}_{12}^T(\lambda_{12})\bm{S}_{12}(\lambda_{12})) - 1) + \frac{\lambda_{21}}{2}(\tr(\bm{S}_{21}^T(\lambda_{21})\bm{S}_{21}(\lambda_{21})) - 1)\\
\end{aligned}
\end{equation}

One of the partial derivatives is
\begin{equation}
\begin{aligned}
\frac{\partial d(\lambda_{12}, \lambda_{21})}{\partial \lambda_{12}} = & -\tr(\bm{X}_1^T\bm{X}_2\bm{S}'_{12}(\lambda_{12})) + \tr(\bm{S}_{12}^T(\lambda_{12})\bm{X}_2^T\bm{X}_2\bm{S}_{12}'(\lambda_{12}))\\
& + \frac{1}{2}(\tr(\bm{S}_{12}^T(\lambda_{12})\bm{S}_{12}(\lambda_{12})) - 1) + \lambda_{12}\tr(\bm{S}_{12}^T(\lambda_{12})\bm{S}'_{12}(\lambda_{12}))\\
= & -\tr(\bm{X}_1^T\bm{X}_2\bm{S}'_{12}(\lambda_{12})) + \tr(\bm{S}_{12}^T(\lambda_{12})(\bm{X}_2^T\bm{X}_2 + \lambda_{12}\bm{I}_{n_2})\bm{S}_{12}'(\lambda_{12}))\\
& + \frac{1}{2}(\tr(\bm{S}_{12}^T(\lambda_{12})\bm{S}_{12}(\lambda_{12})) - 1)\\
= & -\tr(\bm{X}_1^T\bm{X}_2\bm{S}'_{12}(\lambda_{12})) + \tr(\bm{X}_1^T\bm{X}_2\bm{S}_{12}'(\lambda_{12}))\\
& + \frac{1}{2}(\tr(\bm{S}_{12}^T(\lambda_{12})\bm{S}_{12}(\lambda_{12})) - 1)\\
= & \frac{1}{2}(\Vert\bm{S}_{12}(\lambda_{12})\Vert_F^2 - 1)
\end{aligned}
\end{equation}

Similarly, the other partial derivative is
\begin{equation}
\begin{aligned}
\frac{\partial d(\lambda_{12}, \lambda_{21})}{\partial \lambda_{21}} = \frac{1}{2}(\Vert\bm{S}_{21}(\lambda_{21})\Vert_F^2 - 1)
\end{aligned}
\end{equation}
Therefore, the optimality conditions for the tightened MRD can be summarized as follows.
\begin{align}
& \bm{S}_{12}(\lambda_{12}) = (\bm{X}_2^T\bm{X}_2 + \lambda_{12}\bm{I}_{n_2})^{-1}\bm{X}_2^T\bm{X}_1\\
& \bm{S}_{21}(\lambda_{21}) = (\bm{X}_1^T\bm{X}_1 + \lambda_{21}\bm{I}_{n_1})^{-1}\bm{X}_1^T\bm{X}_2\\
& \frac{\partial d(\lambda_{12}, \lambda_{21})}{\partial \lambda_{12}} = \frac{1}{2}(\Vert\bm{S}_{12}(\lambda_{12})\Vert_F^2 - 1) = 0\\
& \frac{\partial d(\lambda_{12}, \lambda_{21})}{\partial \lambda_{21}} = \frac{1}{2}(\Vert\bm{S}_{21}(\lambda_{21})\Vert_F^2 - 1) = 0
\end{align}
which are exactly of Theorem \ref{thm: optimality condition of MRD}.

\end{proof}

\section{Proof for Lemma \ref{lem:upper bdd of regularization coefficient}}
\begin{proof}
If $\Vert\bm{S}_{12}(0)\Vert_2 = \Vert(\bm{X}_2^T\bm{X}_2)^{-1}\bm{X}_2^T\bm{X}_1\Vert_2 = 1$ for $\lambda_{12} = 0$, then the proof is done. Assume $\Vert\bm{S}_{12}(0)\Vert_2 = \Vert(\bm{X}_2^T\bm{X}_2)^{-1}\bm{X}_2^T\bm{X}_1\Vert_2 > 1$, note that $\Vert\bm{S}_{12}(\lambda_{12})\Vert_2 = \Vert(\bm{X}_2^T\bm{X}_2 + \lambda_{12}\bm{I}_{n_2})^{-1}\bm{X}_2^T\bm{X}_1\Vert_2 \to \lambda_{12}^{-1}\Vert\bm{X}_2^T\bm{X}_1\Vert_2$ as $\lambda_{12}$ is large enough. Thus, it suffices to assert that there exists $\lambda'\in\mathbb R_+$ such that $\Vert\bm{S}_{12}(\lambda_{12})\Vert_2 < 1$ for all $\lambda_{12} > \lambda_{12}'$. By intermediate value theorem, there exists $c\in[0,\lambda_{12}']$ such that $\Vert\bm{S}_{12}(c)\Vert_2 = \Vert(\bm{X}_2^T\bm{X}_2 + c\bm{I}_{n_2})^{-1}\bm{X}_2^T\bm{X}_1\Vert_2 = 1$. Since $\nabla_{\lambda_{12}}(\Vert\bm{S}_{12}(\lambda_{12})\Vert_2) \ge 0$, \textit{i.e.}, $\Vert\bm{S}_{12}(\lambda_{12})\Vert_2$ is non-decreasing with respect to $\lambda_{12}$ on $[0,\infty)$, $c$ is also unique. We can give a candidate for such a $\lambda_{12}'$ by the property of triangle inequality of matrix 2-norm. Specifically, assume $\sigma_{min}(\bm{X}_2^T\bm{X}_2)$ is the smallest singular value of $\bm{X}_2^T\bm{X}_2$, since
\begin{equation}
\begin{aligned}
& \Vert\bm{S}_{12}(\lambda_{12})\Vert_2 = \Vert(\bm{X}_2^T\bm{X}_2 + \lambda_{12}\bm{I}_{n_2})^{-1}\bm{X}_2^T\bm{X}_1\Vert_2\\
\le & \Vert(\bm{X}_2^T\bm{X}_2 + \lambda_{12}\bm{I}_{n_2})^{-1}\Vert_2\Vert\bm{X}_2^T\bm{X}_1\Vert_2 = \frac{\Vert\bm{X}_2^T\bm{X}_1\Vert_2}{\sigma_{min}(\bm{X}_2^T\bm{X}_2 + \lambda_{12}\bm{I}_{n_2})}\\
= & \frac{\Vert\bm{X}_2^T\bm{X}_1\Vert_2}{\sigma_{min}(\bm{X}_2^T\bm{X}_2) + \lambda_{12}},
\end{aligned}
\end{equation}
we choose $\lambda_{12}'$ such that $\frac{\Vert\bm{X}_2^T\bm{X}_1\Vert_2}{\sigma_{min}(\bm{X}_2^T\bm{X}_2) + \lambda_{12}'} = 1$ (\textit{i.e.}, $\lambda_{12}' = \Vert\bm{X}_2^T\bm{X}_1\Vert_2 - \sigma_{min}(\bm{X}_2^T\bm{X}_2)$), which immediately implies that $\Vert\bm{S}_{12}(\lambda_{12}')\Vert_2 \le 1$. Therefore, by the preceding part, we can conclude that there exists a unique $\lambda_{12} \in [0,\lambda_{12}']$ such that $\Vert\bm{S}_{12}(\lambda_{12})\Vert_2 = 1$.
\end{proof}

\section{Proof for Lemma \ref{subsec: sensitivity to noise}}

\begin{proof}


First, we derive an upper bound on $d(\Tilde{\bm{X}_1}, \Tilde{\bm{X}_2}) = d(\bm{X}_1 + \bm{\Delta}_1, \bm{X}_2 + \bm{\Delta}_2)$.
Let $(\bm{S}_{12}, \bm{S}_{21})$ be the corresponding coefficient matrices in computing $d(\bm{X}_1, \bm{X}_2)$.
We have the following derivations:
\begin{equation}
\begin{aligned}
& d(\bm{X}_1 + \bm{\Delta}_1, \bm{X}_2 + \bm{\Delta}_2)\\
\le & \sqrt{w_1\Vert(\bm{X}_1 + \bm{\Delta}_1) - (\bm{X}_2 + \bm{\Delta}_2)\bm{S}_{12}\Vert_F^2 + w_2\Vert(\bm{X}_2 + \bm{\Delta}_2) - (\bm{X}_1 + \bm{\Delta}_1)\bm{S}_{21}\Vert_F^2}\\
= & \sqrt{\Vert[\sqrt{w_1}(\bm{X}_1 - \bm{X}_2\bm{S}_{12}) + \sqrt{w_1}(\bm{\Delta}_1 - \bm{\Delta}_2\bm{S}_{12}),\sqrt{w_2}(\bm{X}_2 - \bm{X}_1\bm{S}_{21}) + \sqrt{w_2}(\bm{\Delta}_2 - \bm{\Delta}_1\bm{S}_{21})]\Vert_F^2}\\
= & \Vert[\sqrt{w_1}(\bm{X}_1 - \bm{X}_2\bm{S}_{12}) + \sqrt{w_1}(\bm{\Delta}_1 - \bm{\Delta}_2\bm{S}_{12}),\sqrt{w_2}(\bm{X}_2 - \bm{X}_1\bm{S}_{21}) + \sqrt{w_2}(\bm{\Delta}_2 - \bm{\Delta}_1\bm{S}_{21})]\Vert_F\\
= & \Vert[\sqrt{w_1}(\bm{X}_1 - \bm{X}_2\bm{S}_{12}),\sqrt{w_2}(\bm{X}_2 - \bm{X}_1\bm{S}_{21})] + [\sqrt{w_1}(\bm{\Delta}_1 - \bm{\Delta}_2\bm{S}_{12}), \sqrt{w_2}(\bm{\Delta}_2 - \bm{\Delta}_1\bm{S}_{21})]\Vert_F\\
\le & \Vert[\sqrt{w_1}(\bm{X}_1 - \bm{X}_2\bm{S}_{12}),\sqrt{w_2}(\bm{X}_2 - \bm{X}_1\bm{S}_{21})]\Vert_F + \Vert[\sqrt{w_1}(\bm{\Delta}_1 - \bm{\Delta}_2\bm{S}_{12}), \sqrt{w_2}(\bm{\Delta}_2 - \bm{\Delta}_1\bm{S}_{21})]\Vert_F\\
= & d(\bm{X}_1, \bm{X}_2) + \Vert[\sqrt{w_1}(\bm{\Delta}_1 - \bm{\Delta}_2\bm{S}_{12}), \sqrt{w_2}(\bm{\Delta}_2 - \bm{\Delta}_1\bm{S}_{21})]\Vert_F
\end{aligned}
\end{equation}

Next, we derive an upper bound on the above second term. We have
\begin{equation}
\begin{aligned}
& \Vert[\sqrt{w_1}(\bm{\Delta}_1 - \bm{\Delta}_2\bm{S}_{12}), \sqrt{w_2}(\bm{\Delta}_2 - \bm{\Delta}_1\bm{S}_{21})]\Vert_F\\
= & \Vert[\sqrt{w_1}\bm{\Delta}_1, \sqrt{w_2}\bm{\Delta}_2] - [\sqrt{w_1}\bm{\Delta}_2\bm{S}_{12}, \sqrt{w_2}\bm{\Delta}_1\bm{S}_{21}]\Vert_F\\
\le & \Vert[\sqrt{w_1}\bm{\Delta}_1, \sqrt{w_2}\bm{\Delta}_2]\Vert_F + \Vert[\sqrt{w_1}\bm{\Delta}_2\bm{S}_{12}, \sqrt{w_2}\bm{\Delta}_1\bm{S}_{21}]\Vert_F\\
= & \sqrt{w_1\Vert\bm{\Delta}_1\Vert_F^2 + w_2\Vert\bm{\Delta}_2\Vert_F^2} + \sqrt{w_1\Vert\bm{\Delta}_2\bm{S}_{12}\Vert_F^2 + w_2\Vert\bm{\Delta}_1\bm{S}_{21}\Vert_F^2}\\
\le & \sqrt{w_1\Vert\bm{\Delta}_1\Vert_F^2 + w_2\Vert\bm{\Delta}_2\Vert_F^2} + \sqrt{w_1\Vert\bm{\Delta}_2\Vert_F^2\Vert\bm{S}_{12}\Vert_2^2 + w_2\Vert\bm{\Delta}_1\Vert_F^2\Vert\bm{S}_{21}\Vert_2^2}\\
\le & \sqrt{w_1\Vert\bm{\Delta}_1\Vert_F^2 + w_2\Vert\bm{\Delta}_2\Vert_F^2} + \sqrt{w_1\Vert\bm{\Delta}_2\Vert_F^2 + w_2\Vert\bm{\Delta}_1\Vert_F^2}\\
= & 2\sqrt{w_1\Vert\bm{\Delta}_1\Vert_F^2 + w_2\Vert\bm{\Delta}_2\Vert_F^2}
\end{aligned}
\end{equation}
where we got the first inequality by using the subadditivity of Frobenius norm, got the second inequality based on the definition of operator norm, and got the third inequality by using the constraints $\Vert\bm{S}_{12}\Vert_2 \le 1$ and $\Vert\bm{S}_{21}\Vert_2 \le 1$.

By symmetry, we also have from the above result that
\begin{equation}
\begin{aligned}
d(\bm{X}_1, \bm{X}_2) & \le d(\bm{X}_1 + \bm{\Delta}_1, \bm{X}_2 + \bm{\Delta}_2) + 2\sqrt{w_1\Vert-\bm{\Delta}_1\Vert_F^2 + w_2\Vert-\bm{\Delta}_2\Vert_F^2}\\
& = d(\bm{X}_1 + \bm{\Delta}_1, \bm{X}_2 + \bm{\Delta}_2) + 2\sqrt{w_1\Vert\bm{\Delta}_1\Vert_F^2 + w_2\Vert\bm{\Delta}_2\Vert_F^2}\\
\end{aligned}
\end{equation}
This actually gives a lower bound on $d(\bm{X}_1 + \bm{\Delta}_1, \bm{X}_2 + \bm{\Delta}_2)$, i.e.,
\begin{equation}
\begin{aligned}
d(\bm{X}_1 + \bm{\Delta}_1, \bm{X}_2 + \bm{\Delta}_2) \ge  d(\bm{X}_1, \bm{X}_2) - 2\sqrt{w_1\Vert\bm{\Delta}_1\Vert_F^2 + w_2\Vert\bm{\Delta}_2\Vert_F^2}
\end{aligned}
\end{equation}
In summary, it has a range of
\begin{equation}
\begin{aligned}
\left\vert d(\bm{X}_1 + \bm{\Delta}_1, \bm{X}_2 + \bm{\Delta}_2)- d(\bm{X}_1, \bm{X}_2)\right\vert 
\leq  2\sqrt{w_1\Vert\bm{\Delta}_1\Vert_F^2 + w_2\Vert\bm{\Delta}_2\Vert_F^2}
\end{aligned}
\end{equation}

Since $\bm \Delta_1$ and $\bm \Delta_2$ are both unknown, we can only observe $\Tilde{\bm X}_1$ and $\Tilde{\bm X}_2$. Here, we further deduce an upper bound on the pure noise term $2\sqrt{w_1\Vert\bm{\Delta}_1\Vert_F^2 + w_2\Vert\bm{\Delta}_2\Vert_F^2}$. Since
\begin{equation}
\begin{aligned}
2\sqrt{w_1\Vert\bm{\Delta}_1\Vert_F^2 + w_2\Vert\bm{\Delta}_2\Vert_F^2} & = 2\sqrt{w_1\sigma^2\frac{1}{\sigma^2}\Vert\bm{\Delta}_1\Vert_F^2 + w_2\sigma^2\frac{1}{\sigma^2}\Vert\bm{\Delta}_2\Vert_F^2}\\
& = 2\sqrt{w_1\sigma^2\sum_{i = 1}^{n_1}\sum_{j = 1}^{n_1}\left(\frac{e_{i,j}^{(1)}}{\sigma}\right)^2 + w_2\sigma^2\sum_{i = 1}^{n_2}\sum_{j = 1}^{n_2}\left(\frac{e_{i,j}^{(2)}}{\sigma}\right)^2},
\end{aligned}
\end{equation}
where we know $\frac{1}{\sigma^2}\Vert\bm{\Delta}_1\Vert_F^2 = \sum_{i = 1}^m\sum_{j = 1}^{n_1}\left(\frac{e_{i,j}^{(1)}}{\sigma}\right)^2\sim\chi_{mn_1}^2$ and $\frac{1}{\sigma^2}\Vert\bm{\Delta}_2\Vert_F^2 = \sum_{i = 1}^m\sum_{j = 1}^{n_2}\left(\frac{e_{i,j}^{(2)}}{\sigma}\right)^2 \sim\chi_{mn_2}^2$ by the assumption. According to \cite{laurent2000adaptive}, we know that for any positive $t$, 
\begin{equation}
\begin{aligned}
& \mathbb{P}\left(\frac{1}{\sigma^2}\Vert\bm{\Delta}_1\Vert_F^2 > mn_1 + 2\sqrt{m n_1 t} + 2t\right) \le 1 - e^{-t}\\
& \mathbb{P}\left(\frac{1}{\sigma^2}\Vert\bm{\Delta}_2\Vert_F^2 > mn_2 + 2\sqrt{m n_2 t} + 2t\right) \le 1 - e^{-t}
\end{aligned}
\end{equation}
with $t\in\mathbb R_+$. Thus, one has with probability at least $1 - 2e^{-t} + e^{-2t}$
\begin{equation}
\begin{aligned}
& 2\sqrt{w_1\Vert\bm{\Delta}_1\Vert_F^2 + w_2\Vert\bm{\Delta}_2\Vert_F^2}\\
= & 2\sqrt{w_1\sigma^2\frac{1}{\sigma^2}\Vert\bm{\Delta}_1\Vert_F^2 + w_2\sigma^2\frac{1}{\sigma^2}\Vert\bm{\Delta}_2\Vert_F^2}\\
\le & 2\sqrt{w_1\sigma^2(mn_1 + 2\sqrt{m n_1 t} + 2t) + w_2\sigma^2(mn_2 + 2\sqrt{m n_2 t} + 2t)}\\
\le & 2\left\{w_1\sigma^2(m\max\{n_1,n_2\} + 2\sqrt{m\max\{n_1,n_2\}t} + 2t)\right.\\
& \left.+ w_2\sigma^2(m\max\{n_1,n_2\} + 2\sqrt{m\max\{n_1,n_2\}t} + 2t)\right\}^{1/2}\\
= & 2\sqrt{(w_1 + w_2)\sigma^2\left(m\max\{n_1,n_2\} + 2\sqrt{m\max\{n_1,n_2\}t} + 2t\right)}
\end{aligned}
\end{equation}

Setting $\xi_{m, n_1,n_2} = \sqrt{m\max\{n_1,n_2\} + 2\sqrt{m\max\{n_1,n_2\}t} + 2t}$, one immediately has
\begin{equation}
\begin{aligned}
\left|d(\bm{X}_1 + \bm{\Delta}_1, \bm{X}_2 + \bm{\Delta}_2) - d(\bm{X}_1, \bm{X}_2)\right| \le 2\sigma\xi_{m, n_1,n_2}\sqrt{w_1 + w_2}
\end{aligned}
\end{equation}

\end{proof}

\section{Proof of Lemma \ref{lem_perturbation_error_on_kernel_matrix}}

\begin{proof}
Since $\Tilde{\bm X}_1 = \bm X_1 + \bm \Delta_1$ and $\Tilde{\bm X}_2 = \bm X_2 + \bm \Delta_2$, one has
\begin{equation}\label{eq_error_bdd}
\begin{aligned}
& \left\Vert\mathcal{K}(\Tilde{\bm X}_1, \Tilde{\bm X}_2) - \mathcal{K}(\bm X_1, \bm X_2)\right\Vert_\infty \\
= & \max_{i,j}\left\vert\mathcal{K}(\Tilde{\bm x}_i^{(1)}, \Tilde{\bm x}_j^{(2)})- \mathcal{K}(\bm x_i^{(1)}, \bm x_j^{(2)})\right\vert\\
= & \max_{i,j}\left\vert\exp{\left(-\frac{\left\Vert(\bm x_i^{(1)} + \bm e_i^{(1)}) - (\bm x_j^{(2)} + \bm e_j^{(2)})\right\Vert_2^2}{2r^2}\right)} - \exp{\left(-\frac{\left\Vert\bm x_i^{(1)} - \bm x_j^{(2)}\right\Vert_2^2}{2r^2}\right)} \right\vert \\
\le & \max_{i,j}\frac{1}{2r^2} \left\vert-\left\Vert(\bm x_i^{(1)} + \bm e_i^{(1)}) - (\bm x_j^{(2)} + \bm e_j^{(2)})\right\Vert_2^2 + \left\Vert\bm x_i^{(1)} - \bm x_j^{(2)}\right\Vert_2^2 \right\vert \\
= & \max_{i,j}\frac{1}{2r^2} \left\vert\left\Vert(\bm x_i^{(1)} - \bm x_j^{(2)}) + (\bm e_i^{(1)} - \bm e_j^{(2)})\right\Vert_2^2 - \left\Vert\bm x_i^{(1)} - \bm x_j^{(2)}\right\Vert_2^2 \right\vert \\
= & \max_{i,j}\frac{1}{2r^2} \left\vert\left\Vert\bm e_i^{(1)} - \bm e_j^{(2)}\right\Vert_2^2 + 2\langle \bm x_i^{(1)} - \bm x_j^{(2)}, \bm e_i^{(1)} - \bm e_j^{(2)} \rangle \right\vert \\
\le & \max_{i,j}\frac{1}{2r^2} \left(\left\Vert\bm e_i^{(1)} - \bm e_j^{(2)}\right\Vert_2^2 + 2\left\Vert\bm x_i^{(1)} - \bm x_j^{(2)}\right\Vert_2 \left\Vert\bm e_i^{(1)} - \bm e_j^{(2)}\right\Vert_2\right),
\end{aligned}
\end{equation}
where for the first inequality we have used the fact that the exponential function is locally Lipschitz continuous, i.e., 
\begin{equation*}
|e^x - e^y| < |x - y| \text{ for } x,y < 0.
\end{equation*}
Note that 
\begin{equation}\begin{aligned}
\left\Vert\bm e_i^{(1)} - \bm e_j^{(2)}\right\Vert_2^2 = \sum_{k = 1}^m(e_{k,i}^{(1)} - e_{k,j}^{(2)})^2 = 2\sigma^2\sum_{k = 1}^m\left(\frac{e_{k,i}^{(1)} - e_{k,j}^{(2)}}{\sqrt{2}\sigma}\right)^2
\end{aligned}\end{equation}
for which $e_{k,i}^{(1)}$ and $e_{k,j}^{(2)}$ denote the $(k,i)$-entry of $\bm\Delta_1$ and $(k,j)$-entry of $\bm\Delta_2$, respectively. Since $\frac{e_{k,i}^{(1)} - e_{k,j}^{(2)}}{\sqrt{2}\sigma}\sim\mathcal{N}(0,1)$, we define a random variable here
\begin{equation}\begin{aligned}
Q = \sum_{k = 1}^m\left(\frac{e_{k,i}^{(1)} - e_{k,j}^{(2)}}{\sqrt{2}\sigma}\right)^2\sim\chi_m^2.
\end{aligned}\end{equation}
From \citep{laurent2000adaptive}, one knows that the Chi-squared variable $Q$ satisfies
\begin{equation}\begin{aligned}
\mathbb P(Q > m + 2\sqrt{mt} + 2t) \le 1 - e^{-t}
\end{aligned}\end{equation}
with $t\in\mathbb R_+$. Thus, one can bound $\left\Vert\bm e_i^{(1)} - \bm e_j^{(2)}\right\Vert_2^2$ with probability $1 - e^{-t}$ as
\begin{equation}\begin{aligned}
\left\Vert\bm e_i^{(1)} - \bm e_j^{(2)}\right\Vert_2^2 = 2\sigma^2 Q \le 2\sigma^2(m + 2\sqrt{mt} + 2t)
\end{aligned}\end{equation}
which implies the union bound
\begin{equation}\begin{aligned}
\max_{i,j}\left\Vert\bm e_i^{(1)} - \bm e_j^{(2)}\right\Vert_2^2 = 2\sigma^2 Q \le 2\sigma^2(m + 2\sqrt{mt} + 2t)
\end{aligned}\end{equation}
holds with probability at least $1 - n_1n_2e^{-t}$. Since $\forall i\in[n_1],j\in[n_2], \left\Vert\bm x_i^{(1)} - \bm x_j^{(2)}\right\Vert_2 \le \Vert\mathcal{D}(\bm X_1,\bm X_2)\Vert_\infty$ where $\mathcal{D}(\bm X_1,\bm X_2)$ is the pairwise distance between $\bm X_1$ and $\bm X_2$, and let $\xi_m = \sqrt{m + 2\sqrt{mt} + 2t}$, it follows that with probability at least $1 - n_1n_2e^{-t}$
\begin{equation}\begin{aligned}
\left\Vert\mathcal{K}(\Tilde{\bm X}_1, \Tilde{\bm X}_2) - \mathcal{K}(\bm X_1, \bm X_2)\right\Vert_\infty & \le \frac{1}{2r^2}\left[2\sigma^2\xi_m^2 + 2\Vert\mathcal{D}(\bm X_1,\bm X_2)\Vert_\infty\sqrt{2}\sigma\xi_m\right]\\
& = \frac{1}{r^2}\left[\sigma^2\xi_m^2 + 2\frac{\Vert\mathcal{D}(\bm X_1,\bm X_2)\Vert_\infty}{\sqrt{2}}\sigma\xi_m\right]\\
& = \frac{1}{r^2}\left[\left(\sigma\xi_m + \frac{\Vert\mathcal{D}(\bm X_1,\bm X_2)\Vert_\infty}{\sqrt{2}}\right)^2 - \frac{\Vert\mathcal{D}(\bm X_1,\bm X_2)\Vert_\infty^2}{2}\right]
\end{aligned}\end{equation}
as desired.
\end{proof}

\section{Proof of Theorem \ref{thm_perturbation_analysis_2}}
To derive an upper bound on the perturbed MRD, some necessary lemmas are used here.
\begin{lemma}\label{lem_order_property_1}
Given $\bm A,\bm B\in\mathcal{M}_{m,n}$ and $\bm C\in\mathcal{M}_{n,p}(\mathbb R_+)$, if $\bm A \le \bm B$ (\textit{i.e.}, $\forall i\in[m],j\in[n], A_{i,j}\le B_{i,j}$), then it holds that
\begin{equation}\begin{aligned}
\bm A\bm C \le \bm B\bm C
\end{aligned}\end{equation}
\end{lemma}

\begin{proof}
By the assumption, one can verify $\bm A\bm C - \bm B\bm C = (\bm A - \bm B)\bm C \le \bm O$ where $\bm O\in\mathcal{M}_{m,p}(\mathbb R)$ is the zero matrix.
\end{proof}

\begin{lemma}\label{lem_order_property_2}
Given $\bm A,\bm B\in\mathcal{M}_{m,n}$ and $\bm C\in\mathcal{M}_{p_1,n}(\mathbb R_+),\bm D\in\mathcal{M}_{n,p_2}(\mathbb R_+)$, if $\bm A \le \bm B$ (\textit{i.e.}, $\forall i\in[m],j\in[n], A_{i,j}\le B_{i,j}$), then it holds that
\begin{equation}\begin{aligned}
\bm C\bm A\bm D \le \bm C\bm B\bm D
\end{aligned}\end{equation}
\end{lemma}

\begin{proof}
By the assumption, one can verify $\bm C\bm A\bm D - \bm C\bm B\bm D = \bm C(\bm A - \bm B)\bm D \le \bm O$ where $\bm O\in\mathcal{M}_{p_1,p_2}(\mathbb R)$ is the zero matrix.
\end{proof}

\begin{lemma}\label{lem_order_property_3}
Given $a,b\in\mathbb R_+$,
\begin{equation}\begin{aligned}
\sqrt{a + b} \le \sqrt{a} + \sqrt{b}
\end{aligned}\end{equation}
\end{lemma}

\begin{proof}
Since $(\sqrt{a} + \sqrt{b})^2 - (\sqrt{a + b})^2 = a + b + 2\sqrt{ab} - (a + b) = 2\sqrt{ab} \ge 0$, one immediately has $\sqrt{a} + \sqrt{b} \ge \sqrt{a + b}$ by the monotonicity of $f(x) = \sqrt{x}$.
\end{proof}

\begin{proof}[Proof of Theorem \ref{thm_perturbation_analysis_2}]\hfill

Since
\begin{equation}\begin{aligned}
& d(\bm X_1 + \bm\Delta_1, \bm X_2 + \bm\Delta_2)\\
\le & \sqrt{w_1\Vert\phi(\bm{X}_1 + \bm{\Delta}_1) - \phi(\bm{X}_2 + \bm{\Delta}_2)\bm{S}_{12}\Vert_F^2 + w_2\Vert\phi(\bm{X}_2 + \bm{\Delta}_2) - \phi(\bm{X}_1 + \bm{\Delta}_1)\bm{S}_{21}\Vert_F^2}\\
= & \left\{w_1\text{Tr}(\underbrace{\mathcal{K}(\bm X_1 + \bm\Delta_1,\bm X_1 + \bm\Delta_1)}_{=T_1} \underbrace{- 2\mathcal{K}(\bm X_1 + \bm\Delta_1,\bm X_2 + \bm\Delta_2)\bm S_{12}}_{=T_2} + \underbrace{\bm S_{12}^T\mathcal{K}(\bm X_2 + \bm\Delta_2,\bm X_2 + \bm\Delta_2)\bm S_{12})}_{=T_3}\right.\\
& \hspace{1em}\left.w_2\text{Tr}(\underbrace{\mathcal{K}(\bm X_2 + \bm\Delta_2,\bm X_2 + \bm\Delta_2)}_{=T4} \underbrace{- 2\mathcal{K}(\bm X_2 + \bm\Delta_2,\bm X_1 + \bm\Delta_1)\bm S_{12}}_{=T5} + \underbrace{\bm S_{21}^T\mathcal{K}(\bm X_1 + \bm\Delta_1,\bm X_1 + \bm\Delta_1)\bm S_{21})}_{=T6}\right\}^{1/2},
\end{aligned}\end{equation}
we provide some upper bounds on $T_i,i\in[6]$ as follows.

By Lemma \ref{lem_perturbation_error_on_kernel_matrix}, with probability at least $1 - n_1^2e^{-t}$, one has
\begin{equation}\begin{aligned}
\left\Vert\mathcal{K}\left(\bm X_1 + \bm\Delta_1,\bm X_1 + \bm\Delta_1\right) - \mathcal{K}(\bm X_1,\bm X_1)\right\Vert_\infty \le \varepsilon_{11}
\end{aligned}\end{equation}
where $\varepsilon_{11} = \frac{1}{r^2}\left[\left(\sigma\xi + \frac{\Vert\mathcal{D}(\bm X_1,\bm X_1)\Vert_\infty}{\sqrt{2}}\right)^2 - \frac{\Vert\mathcal{D}(\bm X_1,\bm X_1)\Vert_\infty^2}{2}\right]$. It implies that 
\begin{equation}\begin{aligned}
\mathcal{K}(\bm X_1,\bm X_1) - \varepsilon_{11}\mathbb I_{n_1}\mathbb I_{n_1}^T \le \mathcal{K}(\bm X_1 + \bm\Delta_1,\bm X_1 + \bm\Delta_1) \le \mathcal{K}(\bm X_1,\bm X_1) + \varepsilon_{11}\mathbb I_{n_1}\mathbb I_{n_1}^T
\end{aligned}\end{equation}
Analogously, with probability at least $1 - n_2^2e^{-t}$, $T_4$ admits
\begin{equation}\begin{aligned}
\mathcal{K}(\bm X_2,\bm X_2) - \varepsilon_{22}\mathbb I_{n_2}\mathbb I_{n_2}^T \le \mathcal{K}(\bm X_2 + \bm\Delta_2,\bm X_2 + \bm\Delta_2) \le \mathcal{K}(\bm X_2,\bm X_2) + \varepsilon_{22}\mathbb I_{n_2}\mathbb I_{n_2}^T
\end{aligned}\end{equation}
where $\varepsilon_{22} = \frac{1}{r^2}\left[\left(\sigma\xi + \frac{\Vert\mathcal{D}(\bm X_2,\bm X_2)\Vert_\infty}{\sqrt{2}}\right)^2 - \frac{\Vert\mathcal{D}(\bm X_2,\bm X_2)\Vert_\infty^2}{2}\right]$.

In terms of $T_2$, with probability at least $1 - n_1n_2e^{-t}$, we have by Lemma \ref{lem_perturbation_error_on_kernel_matrix} that 
\begin{equation}\begin{aligned}
\mathcal{K}(\bm X_1,\bm X_2) - \varepsilon_{12}\mathbb I_{n_1}\mathbb I_{n_2}^T \le \mathcal{K}(\bm X_1 + \bm\Delta_1,\bm X_2 + \bm\Delta_2) \le \mathcal{K}(\bm X_1,\bm X_2) + \varepsilon_{12}\mathbb I_{n_1}\mathbb I_{n_2}^T
\end{aligned}\end{equation}
where $\varepsilon_{12} = \frac{1}{r^2}\left[\left(\sigma\xi + \frac{\Vert\mathcal{D}(\bm X_1,\bm X_2)\Vert_\infty}{\sqrt{2}}\right)^2 - \frac{\Vert\mathcal{D}(\bm X_1,\bm X_2)\Vert_\infty^2}{2}\right]$

However, since both $\bm S_{12}$ and $\bm S_{21}$ may have positive or negative entries, it is a little difficult to directly use Lemma \ref{lem_perturbation_error_on_kernel_matrix} to deduce some bounds on $T_2$ and $T_5$. One way introduced here is to use the decomposition that $\bm S = \bm S^+ - \bm S^-$ for any given matrix $\bm S\in\mathcal{M}_{m,n}(\mathbb R)$ where $\bm S^+ = [\max(0,S_{i,j})]_{i\in[m],j\in[n]}$ and $\bm S^- = [\max(0,- S_{i,j})]_{i\in[m],j\in[n]}$. Clearly, both $\bm S^+$ and $\bm S^-$ are nonnegative matrices. Besides, one can verify $[|S_{i,j}|]_{i\in[m],j\in[n]} = |\bm S| = \bm S^+ + \bm S^-$.

That is, we have
\begin{equation}\begin{aligned}
\bm S_{12} = \bm S_{12}^+ - \bm S_{12}^- \quad\text{ and }\quad \bm S_{21} = \bm S_{21}^+ - \bm S_{21}^-
\end{aligned}\end{equation}
where $\bm S_{12}^+,\bm S_{12}^-,\bm S_{21}^+,\bm S_{21}^-$ are all nonnegative matrices.

Then, it follows from Lemmas \ref{lem_perturbation_error_on_kernel_matrix} and \ref{lem_order_property_1} that
\begin{equation}\begin{aligned}
& -2\mathcal{K}(\bm X_1 + \bm\Delta_1,\bm X_2 + \bm\Delta_2)\bm S_{12}\\
= & -2\mathcal{K}(\bm X_1 + \bm\Delta_1,\bm X_2 + \bm\Delta_2)(\bm S_{12}^+ - \bm S_{12}^-)\\
= & -2\mathcal{K}(\bm X_1 + \bm\Delta_1,\bm X_2 + \bm\Delta_2)\bm S_{12}^+ + 2\mathcal{K}(\bm X_1 + \bm\Delta_1,\bm X_2 + \bm\Delta_2)\bm S_{12}^-\\
\le & -2(\mathcal{K}(\bm X_1,\bm X_2) - \varepsilon_{12}\mathbb I_{n_1}\mathbb I_{n_2}^T)\bm S_{12}^+ + 2(\mathcal{K}(\bm X_1,\bm X_2) + \varepsilon_{12}\mathbb I_{n_1}\mathbb I_{n_2}^T)\bm S_{12}^-\\
= & -2\mathcal{K}(\bm X_1,\bm X_2)\bm S_{12}^+ + 2\varepsilon_{12}\mathbb I_{n_1}\mathbb I_{n_2}^T\bm S_{12}^+ + 2\mathcal{K}(\bm X_1,\bm X_2)\bm S_{12}^- + 2\varepsilon_{12}\mathbb I_{n_1}\mathbb I_{n_2}^T\bm S_{12}^-\\
= & -2\mathcal{K}(\bm X_1,\bm X_2)(\bm S_{12}^+ - \bm S_{12}^-) + 2\varepsilon_{12}\mathbb I_{n_1}\mathbb I_{n_2}^T(\bm S_{12}^+ + \bm S_{12}^-)\\
= & -2\mathcal{K}(\bm X_1,\bm X_2)\bm S_{12} + 2\varepsilon_{12}\mathbb I_{n_1}\mathbb I_{n_2}^T\left|\bm S_{12}\right|
\end{aligned}\end{equation}
where we used the fact that for a nonnegative $\bm C$, $\bm B_l\bm C \le \bm A\bm C \le \bm B_u\bm C$ if $\bm B_l \le \bm A \le \bm B_u$.

Analogously, $T_5$ admits
\begin{equation}\begin{aligned}
- 2\mathcal{K}(\bm X_2 + \bm\Delta_2,\bm X_1 + \bm\Delta_1)\bm S_{12} \le -2\mathcal{K}(\bm X_2,\bm X_1)\bm S_{21} + 2\varepsilon_{21}\mathbb I_{n_2}\mathbb I_{n_1}^T\left|\bm S_{21}\right|
\end{aligned}\end{equation}
where $\varepsilon_{21} = \frac{1}{r^2}\left[\left(\sigma\xi + \frac{\Vert\mathcal{D}(\bm X_2,\bm X_1)\Vert_\infty}{\sqrt{2}}\right)^2 - \frac{\Vert\mathcal{D}(\bm X_2,\bm X_1)\Vert_\infty^2}{2}\right]$. By observing that $\Vert\mathcal{D}(\bm X_1,\bm X_2)\Vert_\infty = \Vert\mathcal{D}(\bm X_2,\bm X_1)\Vert_\infty$, we also have $\varepsilon_{12} = \varepsilon_{21}$.

For $T_3$, it follows from Lemma \ref{lem_perturbation_error_on_kernel_matrix} that
\begin{equation}\begin{aligned}
& \bm S_{12}^T\mathcal{K}(\bm X_2 + \bm\Delta_2,\bm X_2 + \bm\Delta_2)\bm S_{12}\\
= & (\bm S_{12}^+ - \bm S_{12}^-)^T\mathcal{K}(\bm X_2 + \bm\Delta_2,\bm X_2 + \bm\Delta_2)(\bm S_{12}^+ - \bm S_{12}^-)\\
= & ((\bm S_{12}^+)^T - (\bm S_{12}^-)^T)\mathcal{K}(\bm X_2 + \bm\Delta_2,\bm X_2 + \bm\Delta_2)(\bm S_{12}^+ - \bm S_{12}^-)\\
= & \underbrace{(\bm S_{12}^+)^T\mathcal{K}(\bm X_2 + \bm\Delta_2,\bm X_2 + \bm\Delta_2)\bm S_{12}^+}_{=T_7} \underbrace{- (\bm S_{12}^+)^T\mathcal{K}(\bm X_2 + \bm\Delta_2,\bm X_2 + \bm\Delta_2)\bm S_{12}^-}_{=T_8}\\
& \hspace{1em} \underbrace{- (\bm S_{12}^-)^T\mathcal{K}(\bm X_2 + \bm\Delta_2,\bm X_2 + \bm\Delta_2)\bm S_{12}^+}_{=T_9} + \underbrace{(\bm S_{12}^-)^T\mathcal{K}(\bm X_2 + \bm\Delta_2,\bm X_2 + \bm\Delta_2)\bm S_{12}^-}_{=T_{10}}
\end{aligned}\end{equation}

For $T_7$, it follows from Lemmas \ref{lem_perturbation_error_on_kernel_matrix} and \ref{lem_order_property_2} that
\begin{equation}\begin{aligned}
& (\bm S_{12}^+)^T\mathcal{K}(\bm X_2 + \bm\Delta_2,\bm X_2 + \bm\Delta_2)\bm S_{12}^+\\
\le & (\bm S_{12}^+)^T(\mathcal{K}(\bm X_2,\bm X_2) + \varepsilon_{22}\mathbb I_{n_2}\mathbb I_{n_2}^T)\bm S_{12}^+\\
= & (\bm S_{12}^+)^T\mathcal{K}(\bm X_2,\bm X_2)\bm S_{12}^+ + \varepsilon_{22}(\bm S_{12}^+)^T\mathbb I_{n_2}\mathbb I_{n_2}^T\bm S_{12}^+
\end{aligned}\end{equation}
Analogously, $T_{10}$ admits
\begin{equation}\begin{aligned}
(\bm S_{12}^-)^T\mathcal{K}(\bm X_2 + \bm\Delta_2,\bm X_2 + \bm\Delta_2)\bm S_{12}^- \le (\bm S_{12}^-)^T\mathcal{K}(\bm X_2,\bm X_2)\bm S_{12}^- + \varepsilon_{22}(\bm S_{12}^-)^T\mathbb I_{n_2}\mathbb I_{n_2}^T\bm S_{12}^-.
\end{aligned}\end{equation}
In terms of $T_8$, one has
\begin{equation}\begin{aligned}
& - (\bm S_{12}^+)^T\mathcal{K}(\bm X_2 + \bm\Delta_2,\bm X_2 + \bm\Delta_2)\bm S_{12}^-\\
\le & - (\bm S_{12}^+)^T(\mathcal{K}(\bm X_2,\bm X_2) - \varepsilon_{22}\mathbb I_{n_2}\mathbb I_{n_2}^T)\bm S_{12}^-\\
= & - (\bm S_{12}^+)^T\mathcal{K}(\bm X_2,\bm X_2)\bm S_{12}^- + \varepsilon_{22}(\bm S_{12}^+)^T \mathbb I_{n_2}\mathbb I_{n_2}^T\bm S_{12}^-\\
\end{aligned}\end{equation}
Analogously, $T_{9}$ admits
\begin{equation}\begin{aligned}
- (\bm S_{12}^-)^T\mathcal{K}(\bm X_2 + \bm\Delta_2,\bm X_2 + \bm\Delta_2)\bm S_{12}^+ \le - (\bm S_{12}^-)^T\mathcal{K}(\bm X_2,\bm X_2)\bm S_{12}^+ + \varepsilon_{22}(\bm S_{12}^-)^T \mathbb I_{n_2}\mathbb I_{n_2}^T\bm S_{12}^+
\end{aligned}\end{equation}
Thus, for $T_{3}$, we have
\begin{equation}\begin{aligned}
& \bm S_{12}^T\mathcal{K}(\bm X_2 + \bm\Delta_2,\bm X_2 + \bm\Delta_2)\bm S_{12}\\
\le & (\bm S_{12}^+)^T\mathcal{K}(\bm X_2,\bm X_2)\bm S_{12}^+ + \varepsilon_{22}(\bm S_{12}^+)^T\mathbb I_{n_2}\mathbb I_{n_2}^T\bm S_{12}^+\\
& - (\bm S_{12}^+)^T\mathcal{K}(\bm X_2,\bm X_2)\bm S_{12}^- + \varepsilon_{22}(\bm S_{12}^+)^T \mathbb I_{n_2}\mathbb I_{n_2}^T\bm S_{12}^-\\
& - (\bm S_{12}^-)^T\mathcal{K}(\bm X_2,\bm X_2)\bm S_{12}^+ + \varepsilon_{22}(\bm S_{12}^-)^T \mathbb I_{n_2}\mathbb I_{n_2}^T\bm S_{12}^+\\
& + (\bm S_{12}^-)^T\mathcal{K}(\bm X_2,\bm X_2)\bm S_{12}^- + \varepsilon_{22}(\bm S_{12}^-)^T\mathbb I_{n_2}\mathbb I_{n_2}^T\bm S_{12}^-\\
= & (\bm S_{12}^+ - \bm S_{12}^-)^T\mathcal{K}(\bm X_2,\bm X_2)(\bm S_{12}^+ - \bm S^-)\\
& + \varepsilon_{22}(\bm S_{12}^+ + \bm S_{12}^-)^T\mathbb I_{n_2}\mathbb I_{n_2}^T(\bm S_{12}^+ + \bm S^-)\\
= & \bm S_{12}^T\mathcal{K}(\bm X_2,\bm X_2)\bm S_{12} + \varepsilon_{22}|\bm S_{12}|^T\mathbb I_{n_2}\mathbb I_{n_2}^T|\bm S_{12}|
\end{aligned}\end{equation}
Analogously, $T_6$ admits
\begin{equation}\begin{aligned}
\bm S_{21}^T\mathcal{K}(\bm X_1 + \bm\Delta_1,\bm X_1 + \bm\Delta_1)\bm S_{21} \le \bm S_{21}^T\mathcal{K}(\bm X_1,\bm X_1)\bm S_{21} + \varepsilon_{11}|\bm S_{21}|^T\mathbb I_{n_1}\mathbb I_{n_1}^T|\bm S_{21}|
\end{aligned}\end{equation}

Thus, we have
\begin{equation}\begin{aligned}
& \Vert\phi(\bm{X}_1 + \bm{\Delta}_1) - \phi(\bm{X}_2 + \bm{\Delta}_2)\bm{S}_{12}\Vert_F^2\\
= & \text{Tr}(T_1 + T_2 + T_3) = \text{Tr}(T_1) + \text{Tr}(T_2) + \text{Tr}(T_3)\\
\le & \text{Tr}(\mathcal{K}(\bm X_1,\bm X_1) + \varepsilon_{11}\mathbb I_{n_1}\mathbb I_{n_1}^T) + \text{Tr}(-2\mathcal{K}(\bm X_1,\bm X_2)\bm S_{12} + 2\varepsilon_{12}\mathbb I_{n_1}\mathbb I_{n_2}^T\left|\bm S_{12}\right|)\\
& + \text{Tr}(\bm S_{12}^T\mathcal{K}(\bm X_2,\bm X_2)\bm S_{12} + \varepsilon_{22}|\bm S_{12}|^T\mathbb I_{n_2}\mathbb I_{n_2}^T|\bm S_{12}|)\\
= & \text{Tr}(\mathcal{K}(\bm X_1,\bm X_1) -2\mathcal{K}(\bm X_1,\bm X_2)\bm S_{12} + \bm S_{12}^T\mathcal{K}(\bm X_2,\bm X_2)\bm S_{12})\\
& + \text{Tr}(\varepsilon_{11}\mathbb I_{n_1}\mathbb I_{n_1}^T + 2\varepsilon_{12}\mathbb I_{n_1}\mathbb I_{n_2}^T\left|\bm S_{12}\right| + \varepsilon_{22}|\bm S_{12}|^T\mathbb I_{n_2}\mathbb I_{n_2}^T|\bm S_{12}|)\\
\end{aligned}\end{equation}
If setting $\varepsilon = \max\{\varepsilon_{11},\varepsilon_{12},\varepsilon_{21},\varepsilon_{22}\}$, we will be able to futher simplify the preceding result as
\begin{equation}\begin{aligned}
\Vert\phi(\bm{X}_1 + \bm{\Delta}_1) - \phi(\bm{X}_2 + \bm{\Delta}_2)\bm{S}_{12}\Vert_F^2 \le \Vert\phi(\bm{X}_1) - \phi(\bm{X}_2)\bm{S}_{12}\Vert_F^2 + \varepsilon\left\Vert|\bm S_{12}|^T\mathbb I_{n_2} + \mathbb I_{n_1}\right\Vert_2^2
\end{aligned}\end{equation}
Analogously, we have
\begin{equation}\begin{aligned}
\Vert\phi(\bm{X}_2 + \bm{\Delta}_2) - \phi(\bm{X}_1 + \bm{\Delta}_1)\bm{S}_{21}\Vert_F^2 \le \Vert\phi(\bm{X}_2) - \phi(\bm{X}_1)\bm{S}_{21}\Vert_F^2 + \varepsilon\left\Vert|\bm S_{21}|^T\mathbb I_{n_1} + \mathbb I_{n_2}\right\Vert_2^2
\end{aligned}\end{equation}
Therefore, 
\begin{equation}\begin{aligned}
& d(\bm X_1 + \bm\Delta_1, \bm X_2 + \bm\Delta_2)\\
\le & \sqrt{w_1\Vert\phi(\bm{X}_1 + \bm{\Delta}_1) - \phi(\bm{X}_2 + \bm{\Delta}_2)\bm{S}_{12}\Vert_F^2 + w_2\Vert\phi(\bm{X}_2 + \bm{\Delta}_2) - \phi(\bm{X}_1 + \bm{\Delta}_1)\bm{S}_{21}\Vert_F^2}\\
\le & \left\{w_1(\Vert\phi(\bm{X}_1) - \phi(\bm{X}_2)\bm{S}_{12}\Vert_F^2 + \varepsilon\left\Vert|\bm S_{12}|^T\mathbb I_{n_2} + \mathbb I_{n_1}\right\Vert_2^2)\right.\\
& \left.+ w_2(\Vert\phi(\bm{X}_2) - \phi(\bm{X}_1)\bm{S}_{21}\Vert_F^2 + \varepsilon\left\Vert|\bm S_{21}|^T\mathbb I_{n_1} + \mathbb I_{n_2}\right\Vert_2^2)\right\}^{1/2}\\
\le & \sqrt{w_1\Vert\phi(\bm{X}_1) - \phi(\bm{X}_2)\bm{S}_{12}\Vert_F^2 + w_2\Vert\phi(\bm{X}_2) - \phi(\bm{X}_1)\bm{S}_{21}\Vert_F^2}\\
& + \sqrt{w_1\varepsilon\left\Vert|\bm S_{12}|^T\mathbb I_{n_2} + \mathbb I_{n_1}\right\Vert_2^2 + w_2\varepsilon\left\Vert|\bm S_{21}|^T\mathbb I_{n_1} + \mathbb I_{n_2}\right\Vert_2^2}\\
= & d(\bm X_1, \bm X_2) + \sqrt{w_1\varepsilon\left\Vert|\bm S_{12}|^T\mathbb I_{n_2} + \mathbb I_{n_1}\right\Vert_2^2 + w_2\varepsilon\left\Vert|\bm S_{21}|^T\mathbb I_{n_1} + \mathbb I_{n_2}\right\Vert_2^2}
\end{aligned}\end{equation}
In the derivation above, we use the Lemma \ref{lem_order_property_3} for the third inequality and get the last inequality by setting both $\bm S_{12}$ and $\bm S_{21}$ to be the values such that $\sqrt{w_1\Vert\phi(\bm{X}_1) - \phi(\bm{X}_2)\bm{S}_{12}\Vert_F^2 + w_2\Vert\phi(\bm{X}_2) - \phi(\bm{X}_1)\bm{S}_{21}\Vert_F^2}$ is exactly $d(\bm X_1, \bm X_2)$ because $\bm S_{12}$ and $\bm S_{21}$ are both arbitrary here.

By symmetry, one immediately has
\begin{equation}\begin{aligned}
d(\bm X_1 + \bm\Delta_1, \bm X_2 + \bm\Delta_2) \ge d(\bm X_1, \bm X_2) - \sqrt{w_1\varepsilon\left\Vert|\bm S_{12}|^T\mathbb I_{n_2} + \mathbb I_{n_1}\right\Vert_2^2 + w_2\varepsilon\left\Vert|\bm S_{21}|^T\mathbb I_{n_1} + \mathbb I_{n_2}\right\Vert_2^2}
\end{aligned}\end{equation}

Nevertheless, since $\bm S_{12}$ and $\bm S_{21}$ are both unknown, we further give a bound on the second term in the last equality:
\begin{equation}
\begin{aligned}
\Vert|\bm S_{12}|^T\mathbb I_{n_2} + \mathbb I_{n_1}\Vert_2^2 & \le (\Vert|\bm S_{12}|^T\mathbb I_{n_2}\Vert_2 + \Vert\mathbb I_{n_1}\Vert_2)^2\\
& \le (\Vert|\bm S_{12}|^T\Vert_F\Vert\mathbb I_{n_2}\Vert_2 + \Vert\mathbb I_{n_1}\Vert_2)^2\\
& = (\Vert \bm S_{12}\Vert_F\Vert\mathbb I_{n_2}\Vert_2 + \Vert\mathbb I_{n_1}\Vert_2)^2\\
& \le (\sqrt{\text{Rank}(\bm S_{12})}\Vert\bm S_{12}\Vert_2\Vert\mathbb I_{n_2}\Vert_2 + \Vert\mathbb I_{n_1}\Vert_2)^2\\
& \le (\sqrt{\min\{n_1,n_2\}n_2} + \sqrt{n_1})^2\\
& \le (\sqrt{\min\{n_1,n_2\}n_2} + \sqrt{\min\{n_1,n_2\}n_1})^2\\
& \le (2\sqrt{\min\{n_1,n_2\}\max\{n_1,n_2\}})^2\\
& = 4 n_1 n_2
\end{aligned}
\end{equation}


Analogously, we have
\begin{equation}
\begin{aligned}
\Vert|\bm S_{21}|^T\mathbb I_{n_1} + \mathbb I_{n_2}\Vert_2^2 \le 4n_1n_2
\end{aligned}
\end{equation}
Thus, 
\begin{equation}
\begin{aligned}
& \sqrt{w_1\varepsilon\Vert|\bm S_{12}|^T\mathbb I_{n_2} + \mathbb I_{n_1}\Vert_2^2 + w_2\varepsilon\Vert|\bm S_{21}|^T\mathbb I_{n_1} + \mathbb I_{n_2}\Vert_2^2}\\
\le & \sqrt{w_1\varepsilon\cdot 4 n_1 n_2 + w_2\varepsilon\cdot 4 n_1 n_2} = 2\sqrt{(w_1 + w_2)\varepsilon n_1 n_2}
\end{aligned}
\end{equation}
Finally, we conclude 
\begin{equation}
\begin{aligned}
\left|d(\bm X_1 + \bm\Delta_1,\bm X_2 + \bm\Delta_2) - d(\bm X_1,\bm X_2)\right| \le \psi_\varepsilon(n_1, n_2)
\end{aligned}
\end{equation}
where $\psi_\varepsilon(n_1, n_2) = 2\sqrt{(w_1 + w_2)\varepsilon n_1 n_2}$.
\end{proof}

\end{document}